\documentclass[10pt]{article}
\usepackage[top=1in, bottom=1in, left=0.85in, right=0.85in]{geometry}

\usepackage[utf8]{inputenc}
\usepackage[T1]{fontenc}

\usepackage{graphicx}
\usepackage[tracking=true,kerning=true]{microtype}
\usepackage{mathtools}
\usepackage{amsthm}
\usepackage{amssymb}
\usepackage[font=footnotesize,labelfont=bf,labelsep=space]{caption}
\usepackage[font=scriptsize]{subcaption}
\usepackage{url}
\usepackage{enumerate}

\usepackage[small]{titlesec}
\usepackage[dvipsnames]{xcolor}

\graphicspath{{./figs/}}
\usepackage[unicode=true,pdfusetitle, bookmarks=true,bookmarksnumbered=false,bookmarksopen=false, breaklinks=false,hidelinks,colorlinks=true,hypertexnames=false]{hyperref}
\hypersetup{colorlinks=true, linkcolor=Magenta,citecolor=Magenta, urlcolor=Magenta}
\usepackage{algorithm}
\usepackage[noend]{algpseudocode}
\usepackage{cleveref}

\DeclareMathOperator*{\diag}{diag}

\usepackage{siunitx}
\usepackage{amsthm}

\newtheorem{theorem}{Theorem}
\newtheorem{prop}{Proposition}
\newtheorem{corollary}{Corollary}
\newtheorem*{objective*}{Objective}

\DeclareMathOperator*{\sech}{sech}
\DeclareMathOperator*{\cov}{cov}
\DeclareMathOperator*{\argmin}{arg\,min}
\newcommand*{\tran}{^{\mathsf{T}}}
\newcommand*{\overbar}[1]{\mkern 1mu\overline{\mkern-1mu#1\mkern-1mu}\mkern 1mu}

\begin{document}

\title{\bfseries A sequential sampling strategy for extreme event statistics in nonlinear dynamical systems}
\author{\normalsize Mustafa A. Mohamad and Themistoklis P. Sapsis\\
\normalsize Department of Mechanical Engineering, Massachusetts Institute of Technology, Cambridge MA 02139}

\maketitle

\begin{abstract}
We develop a method for the evaluation of extreme event statistics associated with nonlinear dynamical systems, using a   small number of samples. From an initial dataset of design points, we formulate a sequential strategy that provides the ‘next-best’ data point (set of  parameters) that when evaluated results in  improved estimates of the probability density function (pdf) for a scalar quantity of interest. The approach utilizes Gaussian process regression to perform  Bayesian inference on the parameter-to-observation map describing the quantity of interest. We then approximate the desired pdf along with uncertainty bounds utilizing the posterior distribution of the inferred  map. The ‘next-best’ design point is sequentially determined through an optimization procedure that selects the point in parameter space that maximally reduces uncertainty between the estimated bounds of the pdf prediction. Since the optimization process utilizes only information  from the inferred map it has minimal computational cost.  Moreover, the special form of the metric emphasizes the tails of the pdf. The method is practical for systems where the dimensionality of the parameter space is of moderate size, i.e. order ${\mathcal{O}(10)}$. We apply the method to estimate the extreme event statistics for a very high-dimensional system with millions of degrees of freedom:  an offshore platform subjected to three-dimensional irregular waves. It is demonstrated that the developed approach can accurately determine the extreme event statistics using limited number of samples.
\end{abstract}

\vspace{1em}\newline
Understanding the statistics of extreme events in dynamical systems of high complexity is of vital importance for reliability assessment and design. We formulate  a new method to pick samples optimally so that we have rapid convergence of the full statistics (complete probability distribution) of a quantity of interest,  including  the tails that describe extreme events. This  is important  for large scale problems in science and engineering where we desire to predict the statistics of relevant quantities but can only afford a limited number of simulations or experiments due to their cost. We demonstrate our approach in a hydromechanical system with millions of degrees of freedom where only 10-20 carefully selected samples can lead to accurate approximation of the extreme event statistics.

\section{Introduction}

For many natural and engineering systems, extreme events, corresponding to large excursions, have significant consequences and are important to predict. Examples include extreme economic events, such as credit shocks~\cite{Fouque2011}, rogue waves in the ocean~\cite{Liu2007}, and extreme climate events~\cite{Hense2006}. Extreme events ‘live’ in the tails of a probability distribution function (pdf), thus it is critical to quantify the pdf many standard deviations away from the mean.  For most real-world problems, the underlying processes are far too complex to enable estimation of the tails through direct simulations or repeated experiments. This is a result of the low probabilities of extreme events, which necessitates a large number of experiments or ensembles to resolve their statistics. For random dynamical systems with inherently nonlinear dynamics (expressed through intermittent events, nonlinear energy transfers, broad energy spectrum, and large intrinsic dimensionality) we are usually limited to a few ensemble realizations. 

The setup in this article involves a stochastic dynamical system that depends on a set of random parameters  with known probability distribution. We assume that the dimensionality of the random parameters is or can be reduced to a moderate size  $\mathcal{O}(10)$. Because of the inherent stochastic and transient character of extreme responses, it is not sufficient to consider the dynamical properties of the system  independently from the statistical characteristics of solutions. A statistical approach to this problem has important limitations, such as requiring various extrapolation schemes due to  insufficient sample numbers (see extreme value theorems~\cite{Nicodemi2012}). Another strategy is large deviations theory~\cite{Varadhan1984,Dematteis2017}, a method for the probabilistic quantification of large fluctuations in systems, which involves identifying a large deviations principle that explains the least unlikely rare event. While  applied to many problems, for complex systems estimating the rate function can be very costly and the principle does not characterize the full probability distribution. The resulting distributions via such approaches cannot always capture the non-trivial shape of the tail, dictated by physical laws in addition to statistical characteristics. On the other hand, in a dynamical systems approach there are no sufficiently generic efficient methods to infer statistical information from dynamics. For example, the Fokker-Planck equation~\cite{Sobczyk1991} is   challenging to solve even in moderate to low dimensions~\cite{Masud2005}. To this end, it is essential to consider blended strategies. The utilization of combined dynamic-stochastic models for the prediction of extreme events have also been advocated and employed in climate science and meteorology by others~\cite{Franzke2017,Majda2016,Chen2017}. In~\cite{Mohamad2016a,Mohamad2015} a probabilistic decomposition of  extreme events was utilized to efficiently characterize the probability distribution of complex systems, which considered both the statistical characteristics of trajectories and the mechanism triggering the instabilities (extreme events). While effective, the proposed decomposition of intermittent regimes requires explicit knowledge of the dynamics  triggering the extremes, which may not be available or easily determined for arbitrary dynamical systems.

We formulate a sequential   method for capturing  the statistics of an observable that is, for example,  a functional of the state of a dynamical system or a physical experiment. The  response of the observable   is modeled using  a machine learning  method that infers the functional form of the  quantity of interest  by utilizing only a few strategically sampled numerical simulations or experiments. Combining the   predictions from the machine learning model, using  the    Gaussian process regression framework,  with  available statistical information on the random input  parameters, we formulate an optimization problem that provides the next-best or most informative experiment that should be performed  to maximally reduce uncertainty in the pdf prediction of  extreme events (tails of the distribution) according to a proposed `distance` metric. To account for tail features  the metric utilize a logarithmic transformation of the pdfs,  which is similar to the style of working on the rate function in the large deviation principle. The optimization process relies exclusively on the inferred  properties on the parameter-to-observation map and no additional simulations are required in the optimization. For the optimization problem to be  practically solvable  we require the parameter space to be of  moderate size, on the    order ${\mathcal{O}(10)}$. The proposed method  allows us  to sequentially   sample  the parameter space in order to rapidly capture the pdf and, in particular, the tails of the distribution of the observable of interest.


\section{Problem setup and method overview}\label{sec:problem_setup}

Consider a dynamical system with state variable $u \in \mathbb{R}^n$, 
\begin{equation}\label{eq:orig_system}
    \frac{du}{dt} =  {g}(t,u;\theta(\omega)),    \quad \omega \in \Omega,
\end{equation}
where $\Omega$ is the sample space in  an  appropriate probability space (we denote the density of the random variable $X$ by $f_X$ and its cumulative density by $F_X$). The random variable $\theta : \Omega \to U \subset  \mathbb{R}^m$ parameterizes sources of uncertainty, such as  stochastic forcing terms  or system parameters with \emph{a priori} known distribution $f_\theta$. For fixed $\omega \in \Omega$, the response $u$ is a deterministic function in time.  We are interested in estimating the pdf $\hat f$  of a scalar quantity of interest or observable $q\in \mathbb{R}$  given by
\begin{equation}
    q = \hat T(\theta)  \triangleq  \mathcal F(u  ) + \varepsilon
\end{equation}
where $\hat T:  U  \subset \mathbb{R}^m  \to \mathbb{R}^1 $ is a   continuous  \emph{parameter-to-observation} map, $\mathcal F$ is an arbitrary  functional of $u$, and $\varepsilon$ is  some   observational or numerical noise term, which we take as zero, without loss of generality. In our setup the unknown parameter-to-observation-map $\hat  T$ is \emph{expensive} to evaluate, representing,  for example,  a large scale numerical simulation or a costly physical experiment, and so we desire to minimize the number of  evaluations of this mapping. Note that the true  statistics of the random variable $q$ induced by $\hat T$ is,
\begin{equation}
    \hat f(s) =  \frac{d}{ds} \hat F(s) = \frac{d}{ds}\mathbb P( \hat T(\theta) \le s  ) = \frac{d}{ds}\int_{A(s)} f_\theta(\theta) \, d\theta,
\end{equation}
where $A(s) = \{ \theta \in U \colon \hat T(\theta) \leq s \}$ and $\hat F$ is the cumulative density function of $q$.

Our objective is to estimate the statistics, especially non-Gaussian features,  of the observable $q$, i.e. the  pdf of the random variable induced by the mapping $\hat  T(\theta)$ which we denote by $\hat f$:
\begin{quote} 
Consider the quantity of interest $q$ with pdf $\hat f$ induced by the unknown mapping $\hat T(\theta)$,  where $\theta$ is a random valued parameter with known pdf $f_\theta$. Given a dataset $\mathcal D_{n-1} = \{( \theta_i, \hat T(  \theta_i)) \}_{i=1}^{n-1}$  of  size $n-1$, so that the estimated  distribution of $q$ using a learning algorithm from this dataset  is $f_{D_{n-1}}$,   determine  the next   input parameter  $\theta$   such that when the map  is evaluated at this new sample  the error between  $f_{D_n}$ and $\hat f$ is minimized,  placing special   emphasis on the tails of the distribution where $ \lvert \hat T(\theta) \rvert$ is large.
\end{quote}
If we consider the augmented $\theta$-parameterized dataset $\mathcal D_n(\theta) = \mathcal D_{n-1} \cup \{(\theta, \hat T(\theta)) \}$ for all $\theta \in U$, and denote the learned density by pdf $f_{\mathcal D_n(\theta)}$, the next parameter $\theta_n$ is obtained by minimizing a distance metric between two probability distributions,  $\theta_n = \argmin_\theta Q(f_{\mathcal D_n(\theta)}, \hat f)$

Determining  the next-best experiment $\theta_{n}$ should  not  involve    the   expensive-to-evaluate map $\hat T$   nor  $\hat f$. The main computational savings of our proposed method involves (1) using a inexpensive-to-compute surrogate model to replace $\hat T(\theta)$ appearing in the $\theta$-parametrized dataset  $D_n(\theta)$ and using a version of the distance metric without explicit dependence on $\hat f$. Here we  utilize  Gaussian process regression (GPR), as the   learning algorithm to construct the surrogate for $\hat T$. Using the posterior distribution of the inferred map through the GPR scheme,  we estimate the pdf for the quantity of interest as well as the pdfs that correspond to the confidence intervals on the map. The distance metric $Q$, is then  based on   minimization of the   logarithmic transform  of the pdfs that correspond to the map upper and lower confidence intervals from the posterior variance, which does not involve $\hat f$.   This optimization problem provides the `next-best' point in a sequential fashion.

The overall  aim is to accurately capture   the tail statistics of $\hat f$ through a minimum number of observations of $q$. Note that the resulting method should not  need to densely sample all regions in $U$, since not all regions have significant probability ($f_\theta$ may be negligible)  or importance ($\vert \hat T\rvert $ may be small). The formulated   sampling strategy should accurately predict the tail region of the pdf  taking into account both the magnitude of the map $\lvert{\hat T}\rvert$ and the value of the probability of the sample $\theta$, see Fig.~\ref{fig:map_overview}. 
\begin{figure}[tbhp!]
    \centering
    \includegraphics[width=0.5\linewidth]{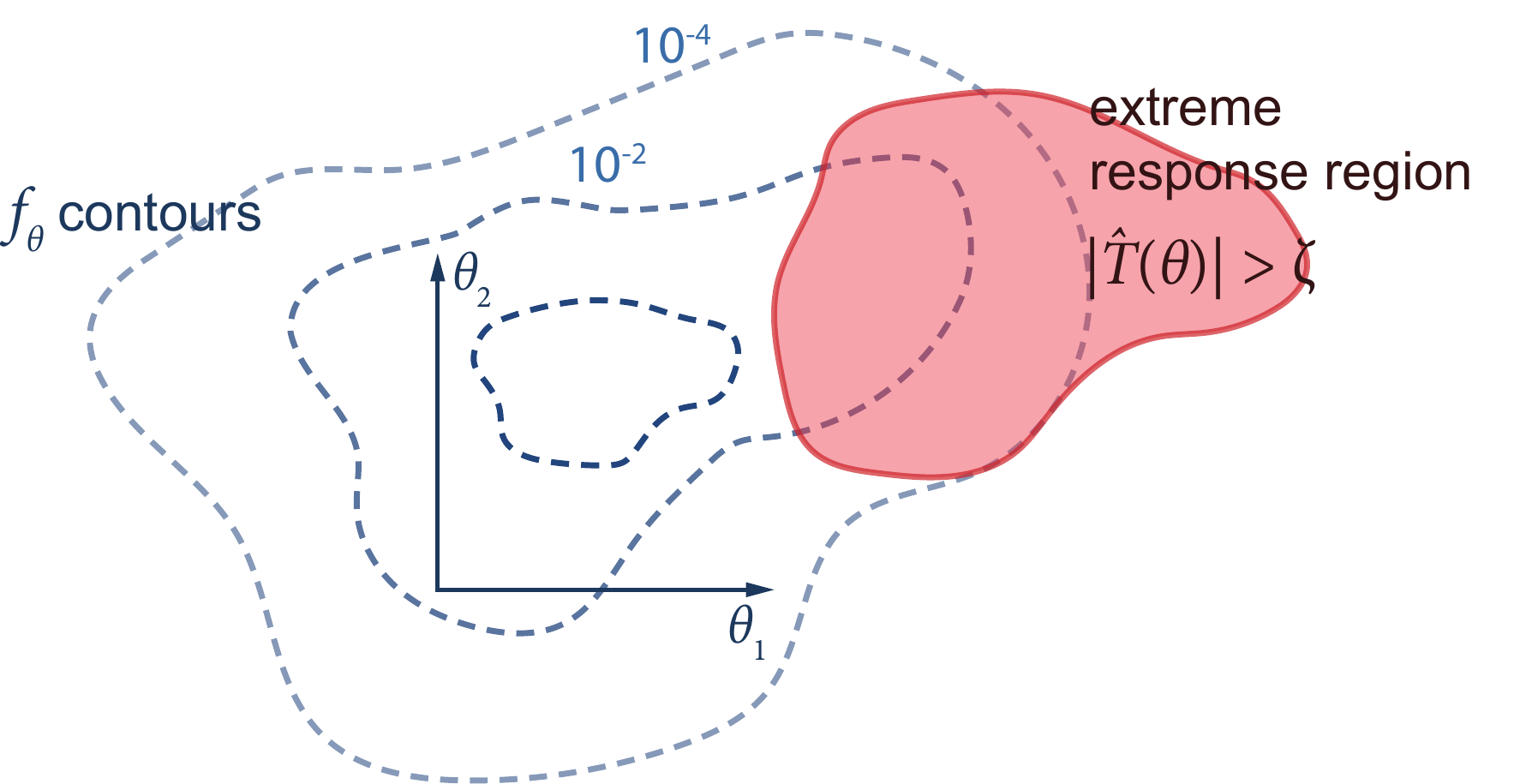}
    \caption{Areas with large probability in $\theta$ are not necessarily associated with regions where $\lvert{\hat T}\rvert $ is large. The proposed criterion focuses on sampling regions where both the  probability  and the   magnitude of~$\lvert{\hat T}\rvert$ are significant.}
    \label{fig:map_overview}
\end{figure}

\section{Method description}\label{sec:algorithm}

An important component of our proposed method is the construction of an inexpensive   surrogate for the map $\hat T$. We utilize  Gaussian process regression (GPR)   as our learning method. GPR   considers the   function  as a Gaussian process, in terms of a distribution in function space  (see \cref{sec:overview_of_gaussian_process_regression} and  numerous references, such as~\cite{Rasmussen2005}). An important property of GPR is that the  posterior distribution is  a Gaussian process with  an explicit mean and kernel function. The variance of the posterior  can   be used as a proxy for the error or uncertainty of the  prediction, which we utilize  along with $f_\theta$ to   guide  selection of the next sample point. 

We learn the  parameter-to-observation map  $\hat T \colon U \subset \mathbb R^m \to \mathbb{R}$  using GPR  from the  dataset $\mathcal D_{n-1} = \{(\hat\theta_i, \hat T(\theta_i))\}_{i=1}^{n-1}$.   The GPR method    then provides analytic  expressions  for the mean $T_{n-1}(\theta)$ and kernel $k_{n-1}(\theta,\theta') $ for the posterior distribution  of the Gaussian random function $\boldsymbol T_{n-1}(\theta) \sim \mathrm{GP}(T_{n-1}(\theta), k_{n-1}(\theta,\theta'))$ (see~\cref{sec:overview_of_gaussian_process_regression} for the expressions). 
We can then construct the following estimate for the pdf $\hat f$  using the posterior mean $T_{n-1}$ of the learned  surrogate model $\boldsymbol T_{n-1}$:
\begin{equation}\label{eq:itercdfs}
    f_{n-1}(s)= \frac{dF_{n-1}}{ds}= \frac{d}{ds}\int_{A_{n-1}(s)}f_\theta(\theta)\,d\theta,
\end{equation}
where $A_{n-1}(s) = \{\theta \in U \colon T_{n-1}(\theta) \leq s \}$ and  $F_{n-1}$ is the    cumulative distribution function.

We now formulate the optimization problem  for the next sample point $\theta^*$. Consider the augmented  $\theta$-parameterized dataset $ \widetilde{\mathcal D}_n(\theta) = \mathcal D_{n-1} \cup \{(\theta,T_{n-1}(\theta))\} $, which  approximates $D_n(\theta) =  \mathcal D_{n-1} \cup \{(\theta,\hat T(\theta))\}$ by using the GPR  mean $T_{n-1}$  instead of  $\hat T$.  Then let $\widetilde{\boldsymbol T}_n(\theta;\theta^*) \sim \mathrm{GP}(\widetilde T_n(\theta;\theta^*) , \widetilde k_n(\theta,\theta';\theta^*)) $ denote the   random function trained on the augmented dataset $\widetilde{\mathcal D}_n(\theta^*)$. The pdf of the random variable  $\hat T(\theta)$, where $\theta \sim f_\theta$, is now replaced by a  $\theta^*$-parameterized random probability measure induced by the Gaussian random function  $\widetilde{\boldsymbol T}_n(\theta;\theta^*)$, where $\theta \sim f_\theta$ and $\theta^*$ is a sample point. Note, that the mean $\widetilde{T}_n(\theta;\theta^*)$ of $\widetilde{\boldsymbol T}_n(\theta;\theta^*)$ in the $\theta$-parameterized dataset is   identical to $T_{n-1}(\theta)$  for all $\theta^*$, since the prediction of the value of the map at the sample point is  given by  the posterior mean  $T_{n-1}(\theta)$ at iteration $n-1$.  

The proposed criterion $Q$ is then based on minimization of a distance metric between the pdfs of the   confidence  bounds of   $\boldsymbol{T}_{n}$. Specifically, let $\tilde f^{\pm}_n(\cdot;\theta^*)$ denote the pdfs of the two  random variables  $\widetilde{T}_n(\theta;\theta^*) \pm \alpha \widetilde\sigma_n(\theta;\theta^*)$, where $ \widetilde \sigma_n(\theta) =  \widetilde k_n(\theta,\theta)$, which are the  upper and lower bounds of the confidence interval  based on the $\alpha$-scaled standard deviation of the posterior distribution. The pdfs  corresponding to the confidence bounds are explicitly given by   
\begin{align}\label{eq:lbpdf}
    \tilde f^{\pm}_n(s;\theta^*) = \frac{d }{ds}\widetilde F_n^\pm{}(s;\theta^*)  = \frac{d}{ds}\int_{A^\pm_{n}(s;\theta^*)}f_\theta(\theta)\,d\theta,
\end{align}
where $A^\pm_{n}(s) = \{ \theta\in U\colon \widetilde{T}_n(\theta;\theta^*) \pm \alpha \widetilde\sigma_n(\theta;\theta^*)\leq s \}$.  We   employ  the 95\% interval bounds, so that the standard deviation is scaled by a factor $\alpha = 1.96$. The cdf $\tilde F^+$ corresponding to  the upper confidence bound $\widetilde{T}_n(\theta;\theta^*) + \alpha \widetilde\sigma_n(\theta;\theta^*)$, is a lower bound for $\widetilde{F}_n$, and we have the relation $\widetilde{F}^+(s) \leq \widetilde{F}(s) \leq \widetilde{F}^-(s)$, for all $s$.  Note, although the map mean  of the Gaussian random function  $\widetilde{\boldsymbol T}_n(\theta;\theta^*)$ based on  the $\theta$-parameterized dataset  is identical to   the posterior mean  $T_{n-1}(\theta)$ at iteration $n-1$, the value of the variance $\widetilde \sigma_n(\theta ;\theta^*)$  now vanishes at the sample point: $\widetilde \sigma_n(\theta^*;\theta^*)=0$.

The  distance metric we propose for the selection of the next sample point is given by
\begin{equation}\label{eq:criterion}
   Q(\tilde f_n^+(\cdot;\theta), \tilde f_n^-(\cdot;\theta))  \triangleq  \frac{1}{2}\int \bigl\lvert{\log( \tilde f_{n}^-(s;\theta)) -\log (\tilde  f_{n}^+(s;\theta)}) \bigr\rvert \,ds,
\end{equation}
where the  integral is computed over the intersection of the two domains that the pdfs    are defined over. The next sample point is then determined by solving the optimization problem:
\begin{equation}    
\theta_n = \argmin_\theta  Q(\tilde f^+(\cdot;\theta), \tilde f^-(\cdot;\theta)) .
\end{equation} 
 This is a $L_1$ based metric of the logarithmic transform of the pdfs. The logarithmic transform of the densities in the criterion  effectively emphasizes extreme and rare events, as we explicitly show in Theorem~\ref{thm:main}.  The   computational savings comes from  the construction of  a criterion $Q$ that avoids $\hat f$  and instead uses $\tilde f_{n}^\pm$,  which  involves evaluating  the GPR emulator $T_{n-1}$ (inexpensive) and an additional GPR prediction.

We make a few comments on the optimization problem and the  sequential algorithm. 
The sequential strategy  is summarized in pseudocode in~\cref{sec:algorithm_pseudocode} For the optimization problem, we  utilize a derivative free method, specifically a particle swarm  optimizer.  The integrations for   the  pdfs are computed explicitly from the definition of the Lebesgue integral, by partitioning the co-domain of map $T$. This is much more efficient compared with a Monte-Carlo approach that would result in a very expensive computational task, as long as the dimensionality of the parameter space is low. For high-dimensional parameter spaces, the computational cost of the integration can become prohibitive and in such cases   an   order reduction  in the parameter space should be first attempted, or  alternatively an  importance sampling algorithm  could be used to compute the integral. In the numerical problems, the optimization problem is practically solvable for low dimensional  parameter spaces $\mathcal{O}(5)$. The starting design plan size $n_s$, if not already provided,  should be small;  a Latin hypercube based sampling plan should be utilized for this purpose. We also recommend  as a pre-training period to process a small number of  iterations using a metric without the logarithmic transform of the densities, such as the following $L_2$ based  metric
\begin{equation}
   Q'(\tilde f_n^+(\cdot;\theta), \tilde f_n^-(\cdot;\theta))  \triangleq  \frac{1}{2}\int \bigl\lvert{ \tilde f_{n}^-(s;\theta)  - \tilde  f_{n}^+(s;\theta)}  \bigr\rvert^2 \,ds,
\end{equation}
in order to  capture the main probability  mass (low order moments) before utilizing the proposed metric that emphasizes extreme and rare events. In addition, it is not necessary to retrain the GPR hyperparameters after every iteration, which can remain fixed after being calibrated from a few iterations. Updating the Gaussian process emulator after the addition of new data points can be done in $\mathcal{O}(n^2)$ if the hyperparameters are fixed, otherwise the GPR emulator must be performed anew in $\mathcal{O}(n^3)$ (see~\cref{sec:gpupdate})\footnote{For low-dimensional $\theta$, since we presume the dataset size is small, the cost difference may be negligible.}. 

\subsection{Asymptotic behavior}
The first theoretical result relates to the  convergence of the proposed method to the true pdf as the number of samples goes to infinity (see~\cref{sec:theory} for the proof). The second result  shows that the  asymptotic form of the criterion is given by (see~\cref{sec:theory} for the proof):
\begin{theorem}\label{thm:main}
    Let $T_{n}(\theta)$ and $\sigma_{n}(\theta)$ from the GPR scheme be sufficiently smooth functions of $\theta$. The   asymptotic behavior of ${Q}$ for large $n$ (ensuring small $\sigma$) and small $\lVert \nabla \sigma\rVert/\lVert \nabla  T\rVert$   is given by
    \begin{align} 
        \widetilde{Q}_n  &\triangleq \frac{1}{2}\int
        \bigl \lvert \log f_{n}^+(s)- \log f_{n}^-(s)\bigr\rvert \, ds\\
        &\approx \int \biggl \lvert \frac{\frac{d}{ds}\mathbb{E}(\sigma _n(\theta) \cdot \boldsymbol 1_{T_{n}(\theta) =       s})}{f_{n}(s)} \biggr\rvert\, ds,\label{eq:symps}
    \end{align}
where $\mathbb{E}$  denotes the   expectation over the probability measure $\mathbb{P}_\theta$.
\end{theorem}
Note that the pdf in the denominator under the integral in~\ref{eq:asymp_result} is a direct implication of our choice to consider the difference between the logarithms of the pdfs in the optimization criterion. The pdf of the parameter-to-observation map   $f_n$ in the denominator of the integrand  guarantees that even values of the parameter $\theta$, where the probability is   low (rare events) are   sampled. This point is clearly demonstrated  by the following  corollary (proof in~\cref{sec:theory}).

\begin{corollary}
 Let $\theta \colon \Omega \to [u_{1}, u_{2}] \subset \mathbb{R}$ be a one-dimensional random variable and in addition to the assumptions of Theorem~\ref{thm:main} we assume that $T_n'(\theta)$ is monotonically increasing function. Then, the asymptotic value of $\widetilde{Q}_n$  for   large $n$ has the following property:
    \begin{equation} 
        \widetilde{Q}_n \gtrsim  \biggl\lvert \int_U \sigma_{n}(\theta) \, d(\log f_{\theta}(\theta))
        + \int_U\sigma_n(\theta) \, d(\log T'(\theta)) +  \sigma_n(u_2)-\sigma_n(u_1) \biggl\rvert .
    \end{equation}
\end{corollary}

Therefore for large value of $n$, $\widetilde Q_n$ bounds (within higher order error) a quantity that consists of boundary terms and two integral terms involving sampling the function $\sigma(\theta)$ over the contours of $\log f_{\theta}(\theta)$ and $\log T'_{n}(\theta)$. Consider the first term on the right, which can be discretized as:
\begin{equation}
    \int_U \sigma(\theta) \, d({\log f_{\theta}(\theta)}) = \lim _{N\to\infty}\sum_{i=1}^N \Delta z  \!\!\!\!\sum_{\{ \theta \colon \log f_{\theta}(\theta) = z_i\}} \sigma(\theta),
\end{equation}where $\{z_i\}^N_{i=1}$ is an equipartition of the range of $\log f_{\theta}$. This summation is summarized in Fig.~\ref{fig:integral} (left). The way the integral is computed guarantees that the integrand $\sigma_{n}(\theta)$  will be sampled even in locations where the pdf $f_{\theta}$ has  small value (rare events) or else the criterion would not converge.

On the other hand, if we had instead chosen a criterion that focused directly on the convergence of low order statistics for $\sigma_{n}(\theta)$, such as the same criterion, but \emph{without the logarithmic transformation of the densities},  we would have
\begin{align}
    \widetilde{Q}'_{n} &\triangleq \frac{1}{2}\int \lvert f^+_n(s) -f^-_n(s)\rvert \,ds \nonumber \geq \biggr\lvert \int_U \sigma_{n}(\theta)   \, d(F_\theta(\theta)) \biggr\rvert,
\end{align}
where $F_\theta$ is the cumulative distribution function of $\theta$. The corresponding integration is shown in Fig.~\ref{fig:integral} (right). In such case, sampling the function $\sigma_{n}(\theta)$ in regions of high probability would be sufficient for the criterion to converge. However, such a strategy would most likely lead to large errors in regions associated with rare events since the sampling will be sparse.
\begin{figure}[htb]
    \centering
    \includegraphics[width=0.7\linewidth]{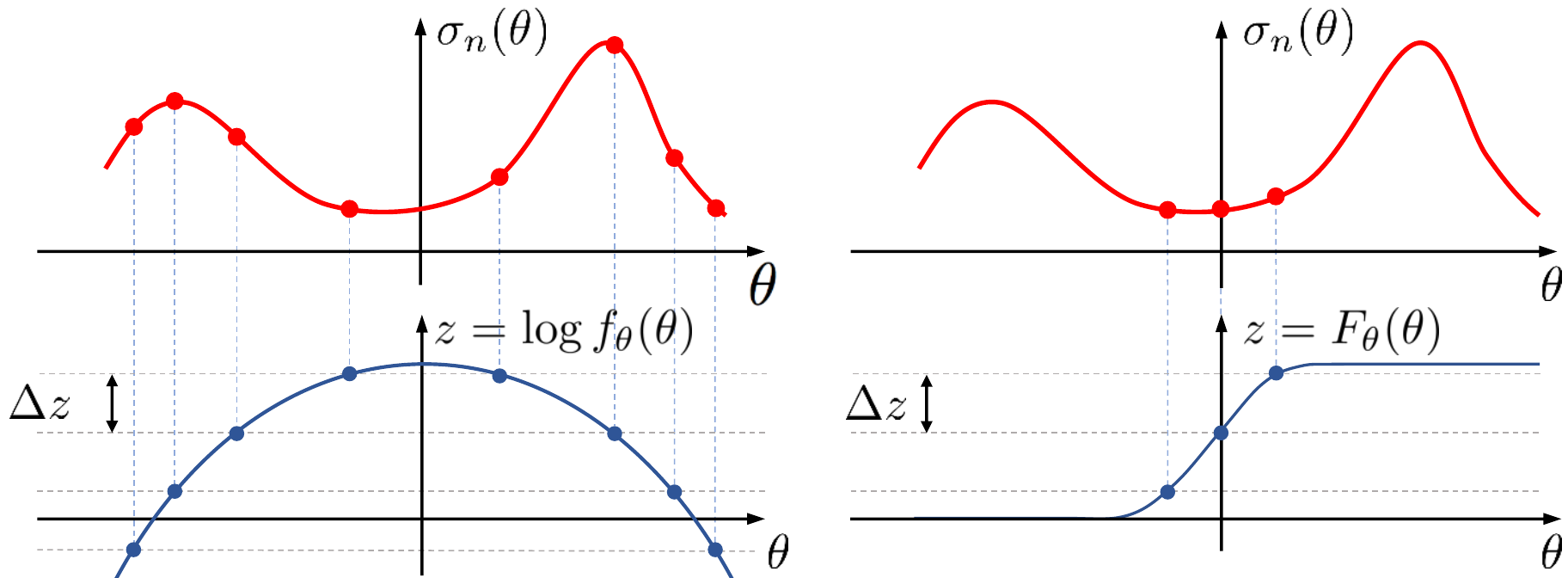}
    \caption{(left) Integration of $\sigma_{n}(\theta)$ over contours of $\log f_{\theta}(\theta)$  implies sampling of $\sigma_{n}(\theta)$   in low probability regions of $\theta$. (right) On the other hand, low-order moments  of $\sigma_{n}(\theta)$  rely only on values of $\sigma_{n}(\theta)$ close to high probability regions of $\theta$, thus rare events are not sampled sufficiently.}
    \label{fig:integral}
\end{figure}

\section{Applications}\label{sec:example_problems}

We illustrate the proposed algorithm to two problems. The first example consists of a nonlinear oscillator stochastically forced by a colored noise process; this application, serves to illustrate the main ideas of the proposed method. The second application, involving three-dimensional  hydrodynamic wave impacts on an offshore platform (a system with millions of degrees of freedom)   showcases the  applicability of the proposed method to real world setups where computation of extreme event statistics using traditional approaches are prohibitively expensive, since the simulation times of  experimental runs are on the order of several hours.

\subsection{Nonlinear oscillator driven by correlated stochastic noise}

Consider the nonlinear oscillator,
\begin{equation}\label{eq:oscillator}
    \ddot x + \delta \dot x + F(x) = \zeta(t),
\end{equation}
forced by a stationary, colored noise with correlation function $C(\tau) = \sigma_\zeta^2 e^{-\tau^2/2\ell_\zeta^2}$ and   the nonlinear restoring term given by
\begin{equation}
    F(x) = \begin{cases}
        \alpha x,                       & 0 \leq \lvert {x} \rvert \leq x_1 \\
        \alpha x_1 ,                    & x_1 < \lvert {x} \rvert \leq x_2  \\
        \alpha x_1 + \beta (x - x_2)^3, & x_2 \leq \lvert {x} \rvert.       \\
    \end{cases}
\end{equation}
Since the system is stochastically forced  it is necessary to use an expansion to obtain a parameterization in terms of a finite number of random variables. We use a Karhunen-Loève expansion (see~\cref{sec:klparam}) to obtain
\begin{equation}\label{eq:dynamoscc}
    \ddot x(t) + \delta \dot x(t) + F(x(t)) = \sum_{i=1}^m \theta_i(\omega) e_i(t), \quad t\in[0, T],
\end{equation}
which is truncated to a suitable number $m$. For illustration, we take our quantity of interest as the average value of the response, so that the parameter-to-observation map is defined by $  T(\theta) \triangleq  \overbar{x(t; \theta)} = \frac{1}{T} \int_0^T x(t;\theta) \, dt.$ 

We consider a three term truncation $m=3$. The system parameters are given by $\delta = 1.5$, $\alpha = 1.0$, $\beta = 0.1$, $x_1 = 0.5$, $x_2 = 1.5$ and the forcing parameters are $\sigma_\zeta = 4$ and $\ell_\zeta = 0.1$, with $t\in[0,25]$. 
For comparisons the exact pdf is obtained by sampling the true map from $64000$ points on a $40\times40\times40$ grid. 
 In Fig.~\ref{fig:n3alg} we illustrate the sampling as determined by the proposed algorithm in addition to the $L_1$ log error between the exact  pdf and the GPR mean prediction. In these simulations we start from a dataset of $6$ points selected according to a Latin-Hypercube (LH) design. In order to   capture the main mass of the pdf, before focusing on the tails of the distribution, we perform $12$ iterations using the $d_{L_2}$ error metric before moving on to the criterion using the $d_{L_1}$ error of the logarithms of the pdfs. Observe the unique shape that the sampling algorithm has identified in $\theta$ space, which   spans regions in $\theta$ associated  with finite probability and large values of $q$. Fig.~\ref{fig:n3alg_prog} demonstrates the progression of the estimated pdf as a function of the iteration count. Even after only $100$ samples we have already captured the qualitative features of the exact pdf and have very good quantitative agreement. 
\begin{figure}[tbhp!]
    \centering
     \includegraphics[scale=0.75]{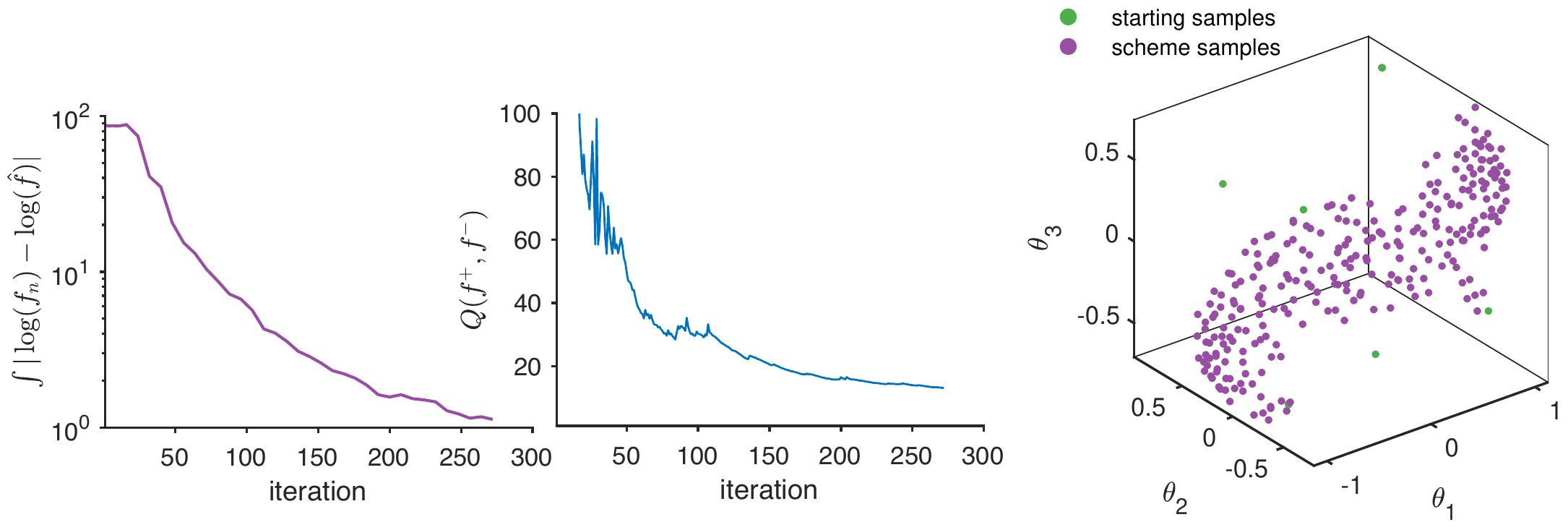}
    \caption{(Right) A scatter plot of the algorithm sampling of the parameter space (green points denote the initial random LH samples). (Left) The corresponding $L_1$ error of the logarithmic transform of the  pdf between the GPR mean and truth and (middle) value of the  criterion in~\ref{eq:lbpdf} as a function of the iteration number.}
    \label{fig:n3alg}
\end{figure}

We have  explored the convergence properties of the algorithm and in Fig.~\ref{fig:n2comparison} we compare the proposed sampling method to space-filling Latin Hypercube sampling. The LH strategy is not iterative and thus must be started anew, which puts the LH sampling at a large disadvantage. Nonetheless, this serves as  a benchmark to a widely used  reference method  for the design of  experiments  due to its simplicity. In the figure, the purple curve represents the mean LH design error and the shaded region represents the standard deviation about the mean, which are computed by evaluating 250 number of random LH designs per fixed dataset size.  Even considering the variance of the LH curve, the proposed algorithm under various parameters (initial dataset size or number of `core' iterations where the $Q$ criterion uses the $L_2$ metric) is observed to outperform the LH strategy by nearly an order of magnitude in the $L_1$ error of the logarithm of the pdfs. This demonstrates the favorable properties of the proposed sampling strategy for accurately estimating the tail of target distribution.  
\begin{figure}[tbhp!]
    \centering
    \includegraphics{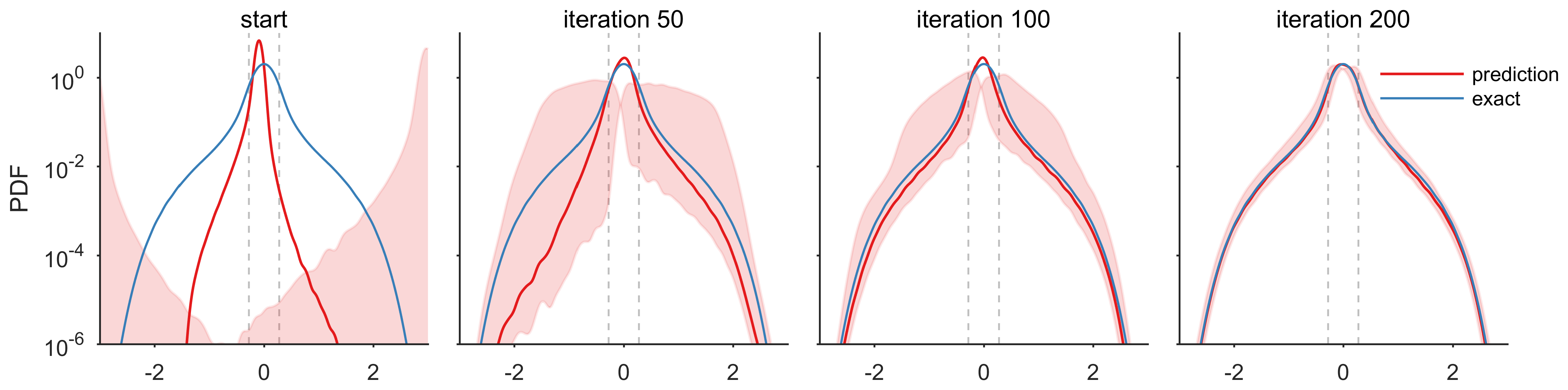}
    \caption{Progression of the pdf estimation as we iteratively sample more points. We shade the region between the pdfs $f^\pm$ in red purely  as a visualization of the convergence of the pdfs $f^\pm$ as more points are sampled.  Dashed vertical lines denote one standard deviation.}
    \label{fig:n3alg_prog}
\end{figure}
\begin{figure}[tbhp!]
    \centering
    \includegraphics[scale=1.25]{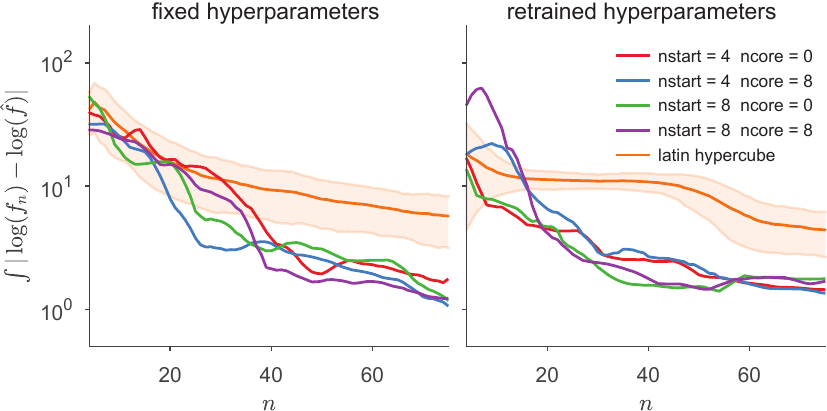}
    \caption{Comparison of the errors between LH sampling and proposed algorithm under different parameters measured against the exact pdf for the case where $m=2$ in~\ref{eq:dynamoscc}. The parameter \texttt{ncore} is the number of `core' iterations performed according to a $L_2$ metric and \texttt{nstart} is initial dataset size (where the points are sampled from an LH design). }
    \label{fig:n2comparison}
\end{figure}

\subsection{Hydrodynamic forces and moments on an offshore platform}

Here we apply the sampling algorithm to compute the probability distributions describing the loads on an offshore platform in irregular seas. The response of the platform is quantified through direct, three-dimensional numerical simulations of Navier-Stokes utilizing the smoothed particle hydrodynamics (SPH) method~\cite{Crespo2015} (Fig.~\ref{fig:sphrend}).  
Our numerical setup parallels that of a physical wave tank experiment and consists of a wave maker on one end and a sloping `beach’ on the other end of the tank to quickly dissipate the energy of incident waves and avoid wave reflections. Further details regarding the simulations are provided in~\cref{sec:offshore}.

Wind generated ocean waves are empirically described by their energy spectrum. Here, we consider irregular seas with JONSWAP spectral density (see~\cref{sec:offshore} for details and parameters). While realizations of the random waves have the form of time series, an alternative description can be obtained by considering a sequence of primary wave groups, each characterized by a random group length scale $L$ and height $A$ (see e.g. ~\cite{Cousins2015}). This formulation allows us to describe the input space through just two random variables (much fewer than what we would need with a Karhunen-Loeve expansion). Following~\cite{Cousins2015} we describe these primary wavegroups by the representation $u(x) = A\sech(x/L),$ which is an explicit parameterization in terms of $L$ and $A$. Thus, $L$ and $A$ correspond to $\theta_1$ and $\theta_2$ in the notation of Eq.~\cref{sec:problem_setup}. The statistical characteristics of the wave groups associated with a random wave field (such as the one given by the JONSWAP spectrum) can be obtained by applying the scale-selection algorithm described in~\cite{Cousins2016}. Specifically, by generating many realizations consistent with the employed spectrum we use a group detection algorithm to identify coherent group structures in the field along with their lengthscale and amplitude ($L$ and $A$). This procedure  provides us with the empirical probability distribution $f_\theta$ of the wave field and thus a nonlinear parametrization of the randomness in the input process.
\begin{figure}[tbhp!]
    \centering
    \includegraphics[width=0.6\linewidth]{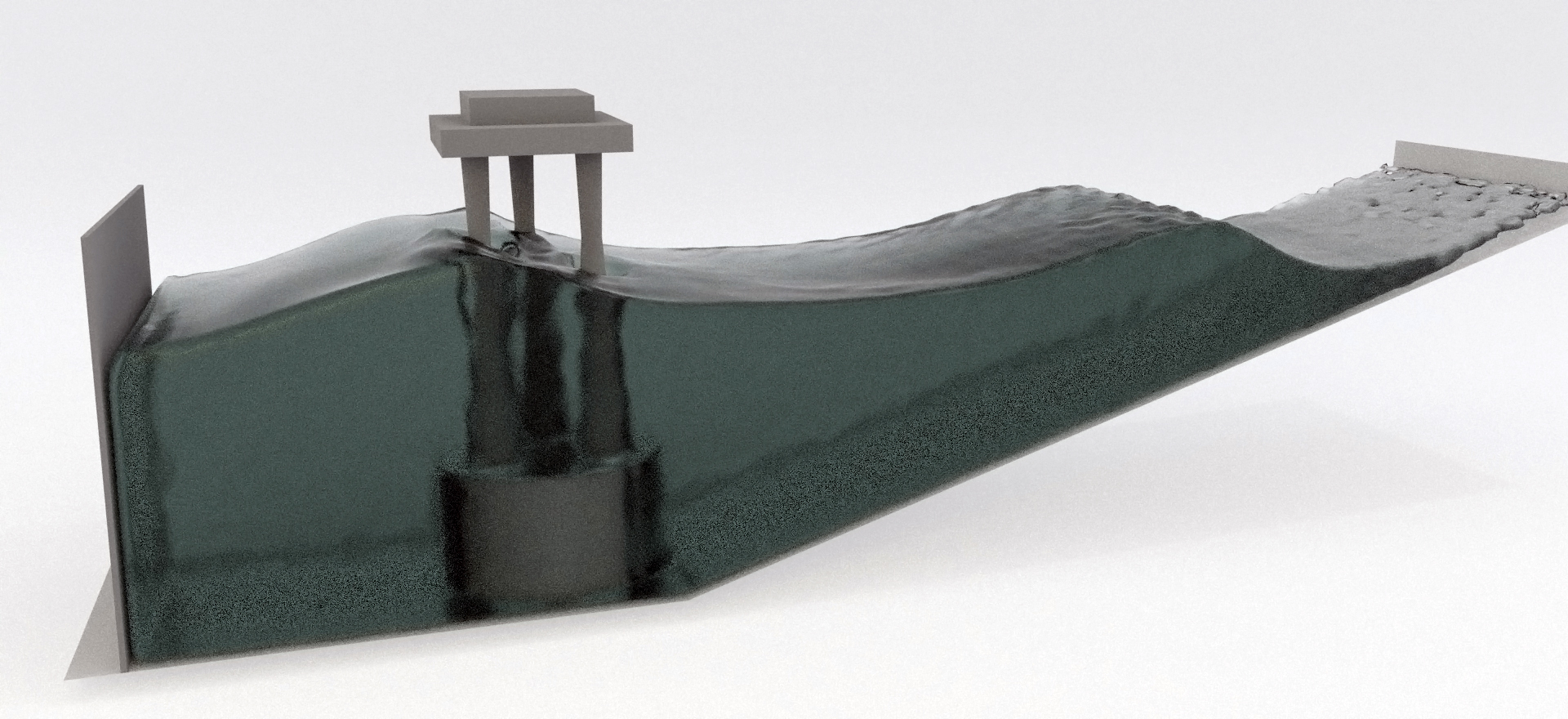}
    \caption{SPH simulation at $t= \SI{103.5}{\second}$ with $\theta_1 = 4.63$ and $\theta_2 = 0.662$.}
    \label{fig:sphrend}
\end{figure}
The quantities of interest in this problem are the forces and moments acting on the platform. The incident wave propagates in the $x$ direction and as such we consider the pdf of the force in the $x$ direction $F_x$ and the moment $M_y$ about the bottom-center of the platform: 
\begin{equation}
q_{f} =  \max_{t\in[0,T]} \lvert{ F_x(t) \rvert} \quad \text{and} \quad q_{m}  =  \max_{t\in[0,T]} \lvert{ M_y(t) \rvert}.
\end{equation}

In Fig.~\ref{fig:sphrend} we show the results of the progression of density prediction for the force variable. In these experiments we begin by arbitrary selecting $4$ initial sample points from a Latin Hypercube sampling strategy. Next, we perform $4$ iterations using the $L_1$ distance metric to quickly capture the main mass of the distribution before focusing on the distribution away from the mean that utilizes the $L_1$ metric of the logarithmic of the pdf.   The lightly shaded red region in the pdf plots is a visualization of the uncertainty in the pdf, obtained by sampling the GPR prediction and computing the pdf for $200$ realizations and then computing the upper (lower) locus of the maximum (minimum) value of the pdf at each value. The figures demonstrate that with $15$ (i.e $14$ total sample points) iterations (together with the $4$ samples in the initial configuration) we are able to approximate the pdf to good agreement with the `exact' pdf, which was computed from a densely sampled grid. In~\cref{sec:offshore} we also present the sampled map where it can be seen that the algorithm selects points associated with large forces and non-negligible probability of occurrence. In the same figures results for the momentum and an additional spectrum are included.  Note, for this problem the GP regression  operates on the logarithm of the observable because the underlying function is always positive.
\begin{figure}[bthp!]
    \centering
    \includegraphics{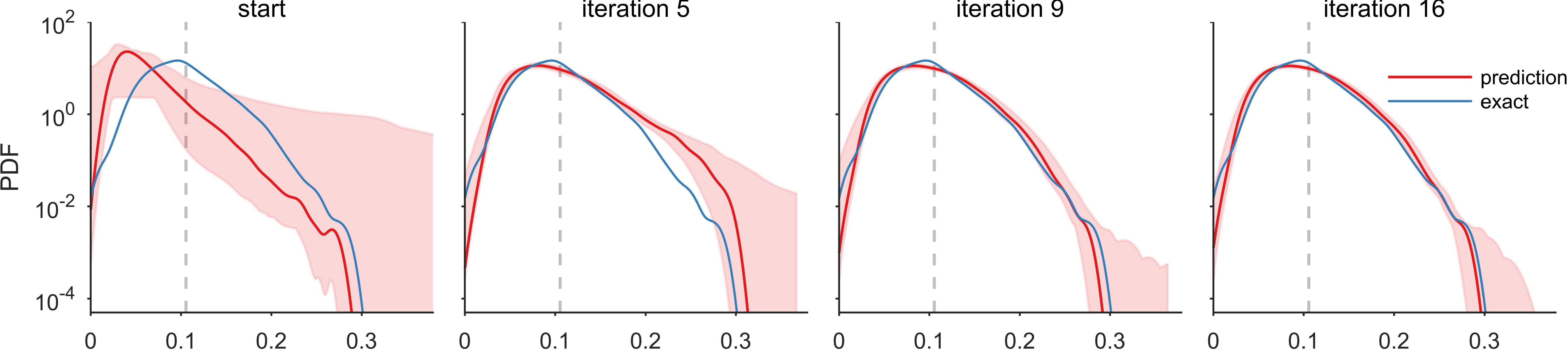}
    \caption{Progression for pdf density prediction for the force variable.}
    \label{fig:result_jonswap1}
\end{figure}

\section{Conclusions}

We developed and analyzed a computational algorithm for the evaluation of extreme event statistics associated with nonlinear dynamical systems that depend on a set of random parameters. The algorithm is practical even for very high dimensional systems but with parameter spaces of moderate dimensionality and it provides a sequence of  points that lead to improved estimates of the probability distribution for a scalar quantity of interest. The criterion for the selection of the next design point emphasizes the tail statistics. We have proven asymptotic convergence of the algorithm and provided  analysis for its asymptotic behavior. We have also demonstrated its applicability through two problems, one of them involving a demanding system with millions degrees of freedom.
\section*{Acknowledgments}
{This work has been supported through the ONR grant N00014-15-1-2381, the AFOSR grant FA9550-16- 1-0231, and the ARO grant W911NF-17-1-0306. We are grateful to the referees for providing numerous suggestions that led to important improvements and corrections. We also thank A. Crespo for helpful comments regarding the SPH simulations.}

\bibliographystyle{plain}
\bibliography{lib}

\begin{thebibliography}{10}

\bibitem{Altomare2015}
Corrado Altomare, Alejandro J.~C. Crespo, Jose~M. Domnguez, Moncho
  Gómez-Gesteira, Tomohiro Suzuki, and Toon Verwaest.
\newblock Applicability of smoothed particle hydrodynamics for estimation of
  sea wave impact on coastal structures.
\newblock {\em Coastal Engineering}, 96:1--12, 02 2015.

\bibitem{Chen2017}
Nan Chen and Andrew~J. Majda.
\newblock Simple stochastic dynamical models capturing the statistical
  diversity of el niño southern oscillation.
\newblock {\em Proceedings of the National Academy of Sciences},
  114(7):1468--1473, 2017.

\bibitem{Cousins2015}
Will Cousins and Themistoklis~P. Sapsis.
\newblock Unsteady evolution of localized unidirectional deep-water wave
  groups.
\newblock {\em Physical Review E}, 91:063204, 2015 06.

\bibitem{Cousins2016}
Will Cousins and Themistoklis~P. Sapsis.
\newblock Reduced-order precursors of rare events in unidirectional nonlinear
  water waves.
\newblock {\em Journal of Fluid Mechanics}, 790:368--388, 02 2016.

\bibitem{Crespo2015}
A.~J.~C. Crespo, J.~M. Domnguez, B.~D. Rogers, M.~Gómez-Gesteira, S.~Longshaw,
  R.~Canelas, R.~Vacondio, A.~Barreiro, and O.~Garca-Feal.
\newblock {DualSPHysics}: Open-source parallel {CFD} solver based on smoothed
  particle hydrodynamics (sph).
\newblock {\em Computer Physics Communications}, 187:204--216, 02 2015.

\bibitem{Davis2005}
Timothy~A. Davis and William~W. Hager.
\newblock Row modifications of a sparse {Cholesky} factorization.
\newblock {\em SIAM Journal on Matrix Analysis and Applications},
  26(3):621--639, 2005.

\bibitem{Dematteis2017}
Giovanni Dematteis, Tobias Grafke, and Eric Vanden-Eijnden.
\newblock Rogue waves and large deviations in deep sea.
\newblock {\em Proceedings of the National Academy of Sciences}, 2018.

\bibitem{Fouque2011}
Jean-Pierre Fouque, George Papanicolaou, Ronnie Sircar, and Knut Sølna.
\newblock {\em Multiscale Stochastic Volatility for Equity, Interest-Rate and
  Credit Derivatives}.
\newblock Cambridge University Press, 09 2011.

\bibitem{Franzke2017}
Christian L.~E. Franzke.
\newblock Extremes in dynamic-stochastic systems.
\newblock {\em Chaos}, 27(1):012101, 2017.

\bibitem{Frigaard1993}
Peter Frigaard, Michael Høgedal, and Morten Christensen.
\newblock Wave generation theory, 06 1993.

\bibitem{Hense2006}
Andreas Hense and Petra Friederichs.
\newblock {\em Wind and precipitation extremes in the {Earth}'s atmosphere},
  pages 169--187.
\newblock Springer Berlin Heidelberg, 2006.

\bibitem{Liu2007}
Paul~C. Liu.
\newblock A chronology of freaque wave encounters.
\newblock {\em Geofizika}, 24(1):57--70, 2007.

\bibitem{Majda2016}
Andrew~J. Majda.
\newblock {\em Introduction to Turbulent Dynamical Systems in Complex Systems}.
\newblock Springer International Publishing, 2016.

\bibitem{Masud2005}
A.~Masud and L.~A. Bergman.
\newblock Solution of the four dimensional {Fokker}-{Planck} equation: still a
  challenge.
\newblock In {\em ICOSSAR}, volume 2005, pages 1911--1916, 2005.

\bibitem{Mohamad2016a}
Mustafa~A. Mohamad, Will Cousins, and Themistoklis~P. Sapsis.
\newblock A probabilistic decomposition-synthesis method for the quantification
  of rare events due to internal instabilities.
\newblock {\em Journal of Computational Physics}, 322:288--308, 10 2016.

\bibitem{Mohamad2015}
Mustafa~A. Mohamad and Themistoklis~P. Sapsis.
\newblock Probabilistic description of extreme events in intermittently
  unstable dynamical systems excited by correlated stochastic processes.
\newblock {\em SIAM/ASA Journal on Uncertainty Quantification}, 3(1):709--736,
  01 2015.

\bibitem{Naess2013}
Arvid Naess and Torgeir Moan.
\newblock {\em Stochastic Dynamics of Marine Structures}.
\newblock Cambridge University Press, 2013.

\bibitem{Nicodemi2012}
Mario Nicodemi.
\newblock {\em Extreme value statistics}, pages 1066--1072.
\newblock Springer New York, 2012.

\bibitem{Papoulis2002}
Athanasios Papoulis and S.~Unnikrishna Pillai.
\newblock {\em Probability, Random Variables and Stochastic Processes}.
\newblock McGraw-Hill Education, 4 edition, 2002.

\bibitem{Pavliotis2014}
Grigorios~A. Pavliotis.
\newblock {\em Stochastic Processes and Applications}, volume~60 of {\em Texts
  in Applied Mathematics}.
\newblock Springer-Verlag New York, 2014.

\bibitem{Rasmussen2005}
Carl~Edward Rasmussen and Christopher K.~I. Williams.
\newblock {\em Gaussian Processes for Machine Learning}.
\newblock Adaptive Computation and Machine Learning. MIT Press, 2005.

\bibitem{Sobczyk1991}
Kazimierz Sobczyk.
\newblock {\em Stochastic Differential Equations}.
\newblock Mathematics and Its Applications (East European Series). Springer
  Netherlands, 1991.

\bibitem{Stuart2016}
Andrew~M. Stuart and Aretha~L. Teckentrup.
\newblock Posterior consistency for gaussian process approximations of bayesian
  posterior distributions.
\newblock {\em ArXiv e-prints}, 2016.

\bibitem{Varadhan1984}
S.~R.~Srinivasa Varadhan.
\newblock {\em Large Deviations and Applications}.
\newblock CBMS-NSF Regional Conference Series in Applied Mathematics. Society
  for Industrial and Applied Mathematics, 1984.

\bibitem{Wendland2004}
Holger Wendland.
\newblock {\em Scattered Data Approximation}, volume~17 of {\em Cambridge
  Monographs on Applied and Computational Mathematics}.
\newblock Cambridge University Press, 2004.

\end{thebibliography}

\appendix

\section{Overview of Gaussian process regression}\label{sec:overview_of_gaussian_process_regression}
An important component of our algorithm is the construction of a surrogate for the map $\hat T$. We utilize the Gaussian process regression (GPR) method for this purpose. A feature of critical importance is that GPR specifies the posterior distribution as a Gaussian random function, with   explicit formulas for the posterior mean and kernel function. The variance  of the posterior can be used as an error or uncertainty estimate of the current prediction, which can in turn be used to guide optimization and explore  parameter space. We briefly provide an overview of GPR since it's a crucial component of our proposed method, but refer to the book~\cite{Rasmussen2005} and numerous other references in the literature  for further details.

We estimate the  parameter-to-observation map, $\hat T(\theta) \colon U \to \mathbb{R}^1,$ with $\ U\in \mathbb{R}^m $, via a GPR scheme from an observed dataset $\mathcal D_n = \{(\theta_i, \hat T(\theta_i))\}_{i=1}^n$, using  $n$ design points. These are the points that we have already sampled. Specifically, to estimate $\hat T$ we place a Gaussian process prior over $\hat T(\theta)$ and consider the function values as a realization of the GP. In particular, with $\Theta = \{\theta_1, \theta_2, \ldots , \theta_n\}$, we have the following posterior mean $T_n(\theta)$ and covariance $k_n(\theta, \theta')$ :
\begin{align}
    T_{n}(\theta)       & = \overbar{T}(\theta) + k_{0}(\theta, \Theta)\tran k_{0}(\Theta,\Theta)^{-1}  (\hat T(\Theta) - \overbar T(\Theta)) \label{eq:mean_gpr} \\
    k_n(\theta,\theta') & = k_0(\theta,\theta') - k_{0}(\theta,\Theta)\tran k_{0} (\Theta,\Theta)^{-1}
    k_0(\Theta,\theta') \label{eq:cov_gpr}
\end{align}
where,

\begin{itemize}
    \item
          $\overbar T(\theta)$ is an arbitrary regression mean function, often chosen to be  a constant or zero,
    \item
           $k_0(\theta,\theta')=\sigma^2 \exp \bigl( - \frac{\lVert \theta - \theta'\rVert^2}{2\lambda} \bigr)$

          is the regression covariance function, here the exponential squared kernel, with $\sigma$ and $\lambda$ being positive hyperparameters,
    \item
          $k_{0}(\Theta,\Theta)\in \mathbb{R}^{n\times n}$ is the covariance matrix, with the $ij^\text{th}$ entry given by $k_{0}( \theta_i, \theta_j)$, and $k_{0}(\theta,\Theta)$, $\overbar T(\Theta),$ $\hat T(\Theta)$ are $n$-dimensional vectors with the $i^\text{th}$ entries given by $k_{0}(\theta, \theta_i),$ $\overbar T( \theta_i)$, $\hat T( \theta_i)$, respectively, and
    \item $\sigma_n^2(\theta) =  k_n(\theta,\theta) $ is  the   variance at $\theta$.
\end{itemize}

There are several important properties to emphasize in the GPR scheme related to the posterior~\cite{Stuart2016}. Firstly, for any choice of the regression function the GPR mean estimate is an interpolant of the exact map at the design points, that is $T_{n}(\Theta) = \hat T(\Theta)$.  Another  property to note in the sequential framework, is that since $k_{0}(\Theta,\Theta)$ is  positive definite:
\begin{equation}
    \sigma_n(\theta)\leq\sigma_{n-1}(\theta)\leq\ldots\leq\sigma_{0}(\theta), \quad \theta \in U,
\end{equation}
 in other words, additional data points lead to non-increasing local variance. Moreover, at each of the design points $\Theta$ the estimated variance vanishes, i.e. $\sigma_n( \theta_i)=0$, for $i=1,\ldots,n$.

We   recall an important result  concerning the convergence of the posterior  mean estimate of $T_n(\theta)$~\cite{Stuart2016,Wendland2004}. This result is used in~\cref{sec:theory} to prove asymptotic convergence of the proposed sampling method.
\begin{prop}\label{prop:mian}
    Let $U\subset\mathbb{R}^M$ be a bounded Lipschitz domain that satisfies an
    interior cone condition. Also, let $T_n(\theta)$ be the mean given by the GPR method.
    Then there exists constants $p>M/2$ and $C$, independent of $\hat
        T(\theta)$, $n$, or $\Theta$ such that
    \begin{equation}
        \lVert T_{n} -\hat T  \rVert_{L^2} \leq  Ch^p_{\Theta}\lVert {\hat T} \rVert_{L^2},
    \end{equation}
    where $h_{\Theta}$ is the fill distance defined as $h_{\Theta} = \sup_{\theta
            \in U}\inf_{  \theta_n \in \Theta} \lVert \theta -  \theta_n \rVert$.
\end{prop}
As shown in~\cite{Stuart2016},   convergence of the mean estimate implies convergence of the estimated variance to zero, i.e.
\begin{equation}
    \sigma_n(\theta) \to 0 \quad\text{as}\quad  n \to \infty \quad \text{and}\quad h_{\Theta} \to 0.
\end{equation}
For uniform tensor grids $\Theta$ the fill distance is of the order $n^{-1/M}$
and thus the convergence of the map to the truth,  in this case, is at least of order $\mathcal{O}(n^{-1/2})$. The precise rate of convergence may differ, depending on the exact criterion used.

\section{Iterative updates to a Gaussian process emulator with fixed hyperparameters}\label{sec:gpupdate}

At every iteration  a new data point  is added to the design matrix $\Theta$. This necessitates refactorization of the covariance matrix $k_n$. Since the Cholesky decomposition in general costs $\mathcal{O}(n^3)$, we would like to update the factorization without recomputing the Cholesky decomposition from scratch at each iteration. We can  update the Cholesky factorization for  a new data point at cost $\mathcal{O}(n^2)$ if the covariance matrix has fixed parameters  (efficient formula's for the removal of a data point  and similarly for block updates can be developed following the same approach, see e.g.~\cite{Davis2005}).

Denote the current covariance matrix with $n-1$ data points  by $C_{11}$ with Cholesky factorization    $C_{11} = A_1 A_1 \tran$ and the covariance matrix after addition of a new design point $\theta_n$ by  $\overbar{C} = \overbar{A} \overbar{A}\tran$. We can write,
\begin{equation}
    A  = \begin{bmatrix} A_1 \\ 0  \end{bmatrix} , \qquad   \overbar{A}  = \begin{bmatrix} A_1 \\ {\overbar{a}_2}\tran  \end{bmatrix},
\end{equation}
  so that
\begin{align}
    C &= \begin{bmatrix} A_1  A_1\tran &  \mathbf{0} \\ \mathbf{0}\tran &  0   \end{bmatrix}  =    \begin{bmatrix} C_{11} &  \mathbf{0} \\ \mathbf{0}\tran  &  0 \end{bmatrix} , \\
    \overbar{C} &= \begin{bmatrix} A_1  A_1\tran &  A_1 {\overbar{a}_2} \\ \overbar{a}_2\tran A_1 &  \overbar{a}_2\tran {\overbar{a}_2}  \end{bmatrix} = \begin{bmatrix} C_{11} &  c_{21} \\ c_{21}\tran  &  c_{22} \end{bmatrix},
\end{align}
where $c_{21} = k_0(\Theta, \theta_n)$ is the covariance of the new data point with the dataset and $c_{22}  =  k_0( \theta_n, \theta_n)$ is variance of the data point. The matrix $\overbar{C}$ can be decomposed as 
\begin{equation}
    \overbar{C} = \begin{bmatrix}
        L_{11}  &  \\ \ell_{21}\tran & \ell_{22}
    \end{bmatrix}
    \begin{bmatrix}
        L_{11}\tran & \ell_{21} \\
                    & \ell_{22}
    \end{bmatrix},
\end{equation}
and thus
\begin{align}
    L_{11}L_{11}\tran                     & = C_{11},                                                              \\
    L_{11}\ell_{21}                       & = c_{21} \implies \ell_{21} = L_{11} \backslash c_{21},                \\
    \ell_{21}\tran\ell_{21} + \ell_{22}^2 & = c_{22} \implies \ell_{22} = \sqrt{c_{22} - \ell_{21}\tran \ell_{21}},
\end{align}
which we  solve for $\ell_{21}$ and $\ell_{22}$ at costs $\mathcal{O}(n^2)$ due to the linear system solve in the  second equation above ($L_{11}$ is already known from the previous iteration).

\section{Algorithm pseudocode}\label{sec:algorithm_pseudocode}
Below we summarize the main loop of the sequential algorithm in pseudocode:
\begin{algorithmic}
    \State \textbf{input} initial dataset $\mathcal{D}_{n_s} = \{( \theta_i, \hat T(\theta_i))\}_{i=1}^{  n_s}$ 
    \Repeat
    \State $T_n,\sigma_n \gets$ predict Gaussian process mean and variance
    \State $f_n, f^+_n, f^-_n \gets$ integration using $T_n$, $\sigma_n$, and $f_\theta$
    \State $\theta^* \gets  \argmin_\theta  Q(\theta, T_n, \mathcal D_n)$
    \State Append $(\theta^*, \hat T(\theta^*))$ to dataset $\mathcal D_n$
    \Until desired error level $d(f^+_n,  f^-_n) < \epsilon$
    \State \textbf{return} $f_{n}$
\end{algorithmic}
This  is an iterative procedure  that leads to a series of pdfs  for the quantity of interest $q$ which, as we show in~\cref{sec:theory}, converges to the true pdf under appropriate conditions. The function  $Q$  that computes the distance criterion  at a test point $\theta$ is summarized below:
\begin{algorithmic}
    \Function{$Q$}{$\theta$, $T_n$, $\mathcal{D}_n$}
    \State Append $(\theta, T_n(\theta))$ to dataset $\mathcal D_n$
    \State $\widetilde \sigma_{n+1} \gets$  predict Gaussian process variance
    \State $\tilde f^+_{n+1}, f^-_{n+1} \gets$ integration using $T_n$, $\widetilde\sigma_{n+1}$, and $f_\theta$
    \State \textbf{return} $d_{L_1}(\log(\tilde f_{n+1}^+), \log(\tilde f_{n+1}^-))$
    \EndFunction
\end{algorithmic}
Above $d$ is  a user defined  distance metric in the stopping condition and $d_{L_1}(f,g) =  \int \lvert{ f-g}\rvert  $.

\section{Convergence and asymptotic behavior of the optimization criterion}\label{sec:theory}

Here we provide details on the proposed sampling algorithm's convergence and  the  asymptotic behavior of the  proposed criterion, which elucidate its favorable extreme event sampling properties. We first prove the convergence of the estimated probability distribution to the true distribution in  the large  sample size $n$ limit. This is under the assumption that the sampling criterion  does not lead to resampling, i.e. we have $h_{\Theta}\to 0 $ as $n\to \infty $. Under this condition the result in~\cref{sec:overview_of_gaussian_process_regression}  guarantees  convergence of the map. We show that $L^2$ convergence of the map implies convergence of the sequence of the transformed pdfs  $f_n$.
\begin{theorem}
    Let the sequence of random variables be denoted by $q_n$ with cumulative distribution function  $F_n$. Moreover,  assume that the sampling criterion does not lead to   resampling of existing  points. Then the $L_2$ convergence of the sequence of maps  $T_n$ to  $\hat T$ by proposition~\ref{prop:mian}   implies  $q_n$ converges in probability to $q$, and this then implies   $F_n$ converges in distribution to $\hat F$. 
\end{theorem}
\begin{proof}
    The assumption that the criterion  does not lead to   resampling of existing samples guarantees   $h_{\Theta}\to0$ as $n\to\infty$ and thus the sequence    $T_n$ converges in $L_2$ to   $\hat T$ by proposition~\ref{prop:mian}.

    We  show that   $L^2$  convergence implies convergence in measure.  For every $\epsilon>0$, by Chebyshev inequality,  we  have
    \begin{align}
         \lVert{T_{n}  - \hat T}\rVert_{L_2}^2 \geq \mathbb E (\lvert T_{n}  - \hat T\rvert^2)   \geq \epsilon^2 \mathbb P(\lvert{ T_n  - \hat T }\rvert \geq \epsilon ).
    \end{align}
    Therefore since the left hand side tends to zero as $n\to\infty$, for every $\epsilon>0$, 
    \begin{displaymath}
        \lim_{n\to\infty} \mathbb P ( \lvert{ T_n  - \hat T }\rvert  \geq \epsilon  ) = 0,
    \end{displaymath}
    thus $T_n$ converges in probability to $\hat T$.  This also implies $T_n$ converges in distribution to $\hat T$: $\lim_{n\to\infty} F_n(s) = \hat F(s)$, for all $s$ that are points of continuity of $\hat F(s)$ (see~\cite{Papoulis2002}).
\end{proof}

The next result is on the asymptotic behavior of the   criterion. Specifically, we show that for large  $n$, so that the distance between the upper and lower map bounds decrease and become small, the proposed criterion $Q$ measures and attempts to minimize the variance of the surrogate map. We show that the metric appropriately weights regions with large  $|T_n|$ and  small  probability $f_\theta$  with  more importance.
\begin{theorem} 
    Let $T_{n}(\theta)$ and $\sigma_{n}(\theta)$ from the GPR scheme be sufficiently smooth functions of $\theta$. The   asymptotic behavior of ${Q}$ for large $n$ (ensuring small $\sigma$) and small $\lVert \nabla \sigma\rVert/\lVert \nabla  T\rVert$   is given by
    \begin{align}\label{eq:asymp_result}
        \widetilde{Q}_n  &\triangleq \frac{1}{2}\int
        \bigl \lvert \log f_{n}^+(s)- \log f_{n}^-(s)\bigr\rvert \, ds\\
        &\approx \int \biggl \lvert \frac{\frac{d}{ds}\mathbb{E}(\sigma _n(\theta) \cdot \boldsymbol 1_{T_{n}(\theta) =       s})}{f_{n}(s)} \biggr\rvert\, ds, 
    \end{align}
    where $\mathbb{E}$  denotes the   expectation over the probability measure $\mathbb{P}_\theta$.
\end{theorem}
\begin{proof}
    To simplify notation we drop the index $n$ in what follows.
    We want to determine the asymptotic behavior of the integral   for $n\to\infty$. In this case, we can assume that $\sigma(\theta)$ is uniformly bounded. We focus on the integrand and determine the  behavior of the difference  $f^+ - f^-$ appearing in the integrand. First focus on  the  cumulative distribution function difference  
    \begin{equation}\label{eq:set_diff}
        F^{+}(s) - F^{-}(s) =  -\int_{A^-(s)\backslash A^+(s)} f_\theta(\theta)\,d\theta,
    \end{equation} where $A^-(s)\setminus  A+(s)$ denotes the   difference between the   sets $A^\pm(s) = \{ \theta \in U \colon T(\theta) \pm \sigma(\theta)\le s\}$.  Now since  $A^-(s)\setminus  A^+(s)$  is equivalent to    $T(\theta) \in (s-\sigma(\theta),s+\sigma(\theta))$,  
    \begin{align}
        F^{+}(s) - F^{-}(s) &=  -\int_{T(\theta) \in (s-\sigma(\theta),s+\sigma(\theta))} f_\theta(\theta)\,d\theta,\\
        &= -\int_{T(\overbar \theta )= s} (h^-(\overbar\theta)- h^+(\overbar\theta))f_\theta(\overbar\theta)\, dS_{\overbar \theta}, \label{eq:integ}
    \end{align}
    where $\int \, dS_{\overbar \theta}$ denotes   surface integration and $\overbar \theta$ is on the surface $T(\overbar \theta) = s$. The functions $h^\pm$ satisfy 
    \begin{equation}\label{eq:hdef}
    T(\overbar\theta + h^\pm(\overbar\theta) \boldsymbol n) \pm \sigma(\overbar\theta + h^\pm(\overbar\theta) \boldsymbol n) = s,
    \end{equation}
    where $\boldsymbol n = \nabla T/\lVert \nabla T \rVert$ is the unit normal   to the surface $T(\overbar \theta) = s$. Note by Taylor theorem we have
    \begin{align}
    T(\overbar\theta + \boldsymbol n t) &= T(\overbar\theta) + (\boldsymbol n\cdot\nabla T(\overbar\theta))  t +    \frac{(\boldsymbol n \cdot \nabla)^2 T( \overbar\theta+\boldsymbol n t_1)}{2}   t^2  ,\\
     \sigma(\overbar\theta + \boldsymbol n t) &= \sigma(\overbar\theta) + (\boldsymbol n\cdot \nabla \sigma( \overbar\theta))   t + \frac{(\boldsymbol n\cdot \nabla)^2 \sigma( \overbar\theta+\boldsymbol n t_1)}{2}   t^2 ,   
    \end{align}
    for $0 < t_1 < t$, and therefore from~\eqref{eq:hdef} we find,
    \begin{equation}
    h^\pm(\overbar \theta) = \frac{ \mp \sigma(\overbar \theta) }{(\nabla T(\overbar \theta)+ \nabla \sigma(\overbar\theta))\cdot \boldsymbol n}  + \mathcal O(\sigma^2).
    \end{equation}
    We then have  
    \begin{align}
       h^- - h^+ \nonumber   &=  \biggl(\frac{  \sigma  }{(\nabla T  - \nabla \sigma)\cdot \boldsymbol n}  +  \frac{  \sigma }{(\nabla T+ \nabla \sigma)\cdot \boldsymbol n}   \biggr) + \mathcal O(\sigma^2)  \\
       &=  \frac{  2  \sigma }{ \lVert\nabla T\rVert \biggl(1 - \dfrac{( \nabla \sigma \cdot \nabla T)^2}{\lVert \nabla T\rVert^3}\biggr)} + \mathcal O(\sigma^2)    \\& = \frac{  2  \sigma }{ \lVert\nabla T\rVert }    + \mathcal O(\tfrac{\lVert \nabla \sigma\rVert}{\lVert \nabla  T\rVert}, \sigma^2) ,
    \end{align}
  where for the second equality we use the Cauchy–Schwarz inequality  $   \nabla \sigma \cdot \nabla T    \le \lVert  \nabla \sigma\rVert  \lVert \nabla T\rVert$ to bound the error in the last line. 
          From this result and~\eqref{eq:integ} we obtain,
    \begin{align}
   f^+ -f^-  &=  -\frac{d}{ds}\int_{T(\overbar \theta )= s} (h^-(\overbar\theta)- h^+(\overbar\theta))f_\theta(\overbar\theta)\, dS_{\overbar \theta}\\
   &= -\frac{d}{ds} \int_{T(\overbar \theta )= s}  \frac{  2  \sigma(\overbar\theta) }{ \lVert\nabla T(\overbar\theta)\rVert } f_\theta(\overbar\theta)\, dS_{\overbar \theta}+ \mathcal O(\tfrac{\lVert \nabla \sigma\rVert}{\lVert \nabla  T\rVert}, \sigma^2)\\
   &=  -\frac{d}{ds} \int_{T(\overbar \theta )= s}    2  \sigma(\overbar\theta)  f_\theta(\overbar\theta)\, d\overbar\theta + \mathcal O(\tfrac{\lVert \nabla \sigma\rVert}{\lVert \nabla  T\rVert}, \sigma^2)\\
   &\approx -2\frac{d}{ds}\mathbb{E}(\sigma(\theta) \cdot \boldsymbol 1_{T(\theta)= s})
    \end{align}
    Next, note the following approximation  
    \begin{equation}
    \log f^+ - \log f^- = \log (1 + \Delta f /f^-) \approx  \Delta f / f^-  \approx \Delta f/ f, 
    \end{equation}
    where $\Delta f = f^+ - f^-$, under the small $\sigma$ assumption. We therefore have the following result on the   behavior of the difference between $f^+$ and $f^-$ 
    \begin{equation}
          \log f^+(s) - \log f^-(s)    \approx    -2\frac{\frac{d }{ds }\mathbb{E}(\sigma(\theta) \cdot \boldsymbol 1_{T(\theta)= s}) }{f(s)},
     \end{equation}
which implies~\eqref{eq:asymp_result}.
\end{proof}

Note that the pdf in the denominator in the integrand~\eqref{eq:symps} is a direct implication of our choice to consider the difference between the logarithmic transform of the pdfs. This   guarantees that  values of the parameter $\theta$ where the probability is small (rare events) are sampled sufficiently  to capture  tail behavior of the map's pdf. This is seen   clearly in   the following  corollary,  where we consider the  special case of a one-dimensional parameter space and a family of maps which are monotonically increasing.

\begin{corollary}
 Let $\theta \colon \Omega \to [u_{1}, u_{2}] \subset \mathbb{R}$ be a one-dimensional random variable and in addition to the assumptions of Theorem~\ref{thm:main} we assume that $T_n'(\theta)$ is monotonically increasing function. Then, the asymptotic value of $\widetilde{Q}_n$  for   large $n$ has the following property:
    \begin{equation}\label{eq:coreq}
        \widetilde{Q}_n \gtrsim \lvert \int_U \sigma_{n}(\theta) \, d(\log f_{\theta}(\theta))
        + \int_U\sigma_n(\theta) \, d(\log T'(\theta)) \\+  \sigma_n(u_2)-\sigma_n(u_1) \rvert .
    \end{equation}
\end{corollary}
\begin{proof}
    For simplicity, we drop the $n$ index in the following. From the proof of Theorem~\ref{thm:main} and using standard inequalities we obtain
\begin{equation}
 \widetilde{Q} \approx \int{\left\vert  \frac{\frac{d}{ds} \int_{T(\theta)=s} \sigma(\theta) f_{\theta}(\theta) d\theta}{f(s)}  \right\vert}ds  
  \geq \left\vert \int{\frac{\frac{d}{ds} \int_{T(\theta)=s} \sigma(\theta) f_{\theta}(\theta) d\theta}{f(s)}}ds \right\vert 
  =  \left\vert \int{\frac{\frac{d^2}{ds^2} \int_{T(\theta) \leq s} \sigma(\theta) f_{\theta}(\theta) d\theta}{\frac{d}{ds} \int_{T(\theta) \leq s} f_{\theta}(\theta) d\theta}}ds \right\vert 
\end{equation}
 Since we assume  $T_n'(\theta)>0$  by applying the transformation $\theta=T^{-1}(s)$ to the numerator of the right hand side  we obtain
    \begin{align}
        \frac{d^2}{ds^2} \int_{T(\theta) \le s} f_\theta(\theta)   \sigma(\theta) \, d\theta & = \frac{d^2}{ds^2} \int^{s} \frac{f_\theta(T^{-1}(s)) \sigma(T^{-1}(s))}{T'(T^{-1}(s))} \, ds                                                                        \\
          & =\frac{d }{ds }\biggl(\frac{f_\theta(T^{-1}(s))\sigma(T^{-1}(s))}{T^{'}(T^{-1}(s))}\biggr) ,  
    \end{align}
    which is equivalent to
    \begin{equation}
      \frac{f'_\theta(T^{-1}(s)) \sigma(T^{-1}(s))}{(T'(T^{-1}(s)))^2}    + \frac{f_\theta(T^{-1}(s))}{T'(T^{-1}(s))}\biggl(\frac{\sigma(T^{-1}(s))}{T'(T^{-1}(s))}\biggr)'. 
    \end{equation}
    Similarly, we have the following for the denominator term:
    \begin{align}
        \frac{d}{ds} \int_{T(\theta)\le s} f_\theta(\theta) \, d\theta   = \frac{d}{ds} \int_{s} \frac{f_\theta(T^{-1}(s))}{T'(T^{-1}(s))} \, ds  
                                                                        = \frac{f_\theta(T^{-1}(s))}{T'(T^{-1}(s))}.
    \end{align}
    Therefore we obtain,
    \begin{align}
        \widetilde{Q} & \gtrsim \left\vert \int_{T(u_1)}^{T(u_2)} \biggl[\frac{f'_\theta(\theta)}{f_\theta(\theta)}\frac{\sigma(\theta)}{T'(\theta)} + \biggl(\frac{\sigma(\theta)}{T'(\theta)}\biggr)' \biggr] _{\theta = T^{-1}(s)} \, ds \right\vert  
                   =  \left\vert \int_{u_1}^{u_2}\biggl[\frac{f'_\theta(\theta)}{f_\theta(\theta)}\sigma(\theta) - \frac{T''(\theta)}{T^{'}(\theta)}\sigma(\theta)+\sigma'(\theta)\biggr] \, d\theta  \right\vert.
    \end{align}
    For each of the first two terms we apply the transformations $z=\log f_{\theta}(\theta)$ and $z = \log T'(\theta)$ respectively to obtain the final result:
    \begin{equation}
        \widetilde{Q} \gtrsim \left\vert \int_{\log f_{\theta}(\theta) = z}\sigma(\theta) \, dz +  \int_{\log T'(\theta) = z} \sigma(\theta) \, dz + \sigma(u) |_{u_1}^{u_2} \right\vert .
    \end{equation}
\end{proof}

\section{Parameterization of a random process using the Karhunen-Loève expansion}\label{sec:klparam}

To apply the developed algorithm  it is important to have a parameterization of the uncertainty in the model; in other words a description of the stochastic forcing, random parameters, or initial conditions in terms of a finite set of random variables.   There are a plethora of methods that can be used to achieve this parameterization with a minimal number of random parameters, which include linear methods such as the Karhunen-Loève   expansion,  also known as principal component analysis,  and nonlinear methods, e.g. kernel principal component analysis.

For completeness we describe the Karhunen-Loève (KL) expansion (see e.g.~\cite{Pavliotis2014,Sobczyk1991}) as one strategy, which is commonly applied in practice due to its simplicity and general applicability. The KL expansion provides an optimal linear parameterization of a random process in the mean square sense. An important caveat is that the KL expansion is not the only strategy, and  might  be ill-suited for certain systems  depending on the  structure of the stochastic process, and thus a nonlinear parameterization may be   better suited to capture the desired  physics with a lower dimensional parameterization.

Consider an $L^2$ process $\zeta(t)$, $0\le t \le \tau$, with mean $\mathbb{E}(\zeta(t))$ and continuous covariance $k(t_1,t_2)$, we can expand this in terms of a double orthogonal process \begin{equation}
    \zeta(t;\omega) = \mathbb{E}(\zeta(t)) + \sum_{i=1}^{\infty} Z_i(\omega)
    e_i(t), \quad 0\le t \le
    \tau,
\end{equation}
where the eigenfunctions $e_i(t)$ of the covariance are orthonormal and the random variables $Z_i(\omega)$ are orthogonal with mean zero and variance $\lambda_i$, that is $\mathbb{E}(Z_i Z_j) = \lambda_i \delta_{ij}$. Through this procedure we can obtain a desired parameterization of the noise in terms of the orthogonal random coefficients of the expansion,
\begin{equation}
    \theta(\omega) = ( Z_1(\omega), \cdots, Z_n(\omega) ),
\end{equation}
where the random variable $\theta \in \mathbb{R}^n$ has pdf $f_\theta$.

Suppose now that  $\zeta(t)$ is a Gaussian process,  we can  then show that the random variables $Z_i$ are uncorrelated and Gaussian, and are therefore independent. Thus, for a mean zero Gaussian process with covariance $k(t_1,t_2) = \cov(\zeta(t_1), \zeta(t_2))$ its KL expansion is given by
\begin{equation}
    \zeta(t) = \sum_{i=1}^{\infty} \sqrt{\lambda_i} \xi_i e_i(t),
\end{equation}
where $\{\lambda_i,e_i(t)\}$ are the eigenvalues and eigenfunctions of the covariance function and $\xi_i \sim \mathcal{N}(0,1)$ are draws from the normal distribution. The variances $\lambda_i >0$ are ordered $\lambda_1 >\lambda_2 > \cdots >0$ so that $\lambda_1$ corresponds to the eigenspace explaining the largest variance. We    thus obtain an $n$-dimensional parameterization of the process by truncating the series to a desired accuracy.
For such a  case,  we  have  $\theta(\omega) = (\sqrt{\lambda_1} \xi_1(\omega), \cdots, \sqrt{\lambda_n} \xi_n(\omega) ) $ and $f_\theta = \mathcal{N}(0, \diag(\lambda_1,\ldots,\lambda_n))$.

\section{Hydrodynamic forces and moments on an offshore platform}\label{sec:offshore}

\subsection{Numerical experiments}
The numerical simulations are performed using the open-source code DualSPHysics~\cite{Crespo2015}, which utilizes the smoothed particle hydrodynamics (SPH) framework, a meshless Lagrangian method. DualSPHysics has been validated on numerous test cases in offshore engineering applications, including forces on structures and also wave propagation, see e.g.~\cite{Altomare2015}.  
\begin{figure}[htb!]
    \centering
    \includegraphics[width=0.75\linewidth]{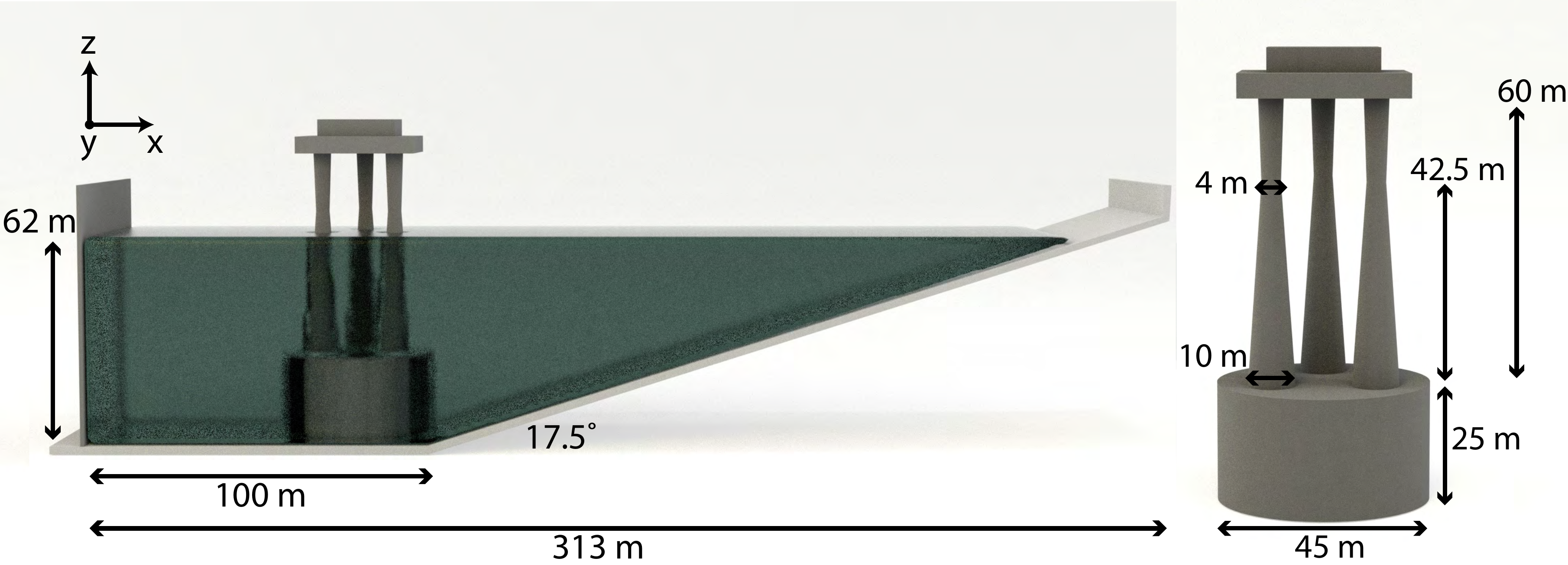}
    \caption{The numerical domain (left) and a close up view of the offshore platform (right).}
    \label{fig:domain}
\end{figure}

A sketch of the numerical domain is provided  in Fig.~\ref{fig:domain}. We are interested in analyzing the forces and moments on offshore structures in the deep ocean, where the dispersion relation is given by  $\omega_0^2 = g k_0$ . We consider  waves with peak period~\SI{10}{\second} so that the characteristic  wavelength is  $\lambda_0 = \frac{2\pi}{k_0} = \SI{156}{\meter}$. In addition, the depth of the wave tank is selected so that $\tanh(k_0 h) = 0.99$, thus the water depth is $h = \SI{62}{\meter}$. The  beach is setup at a~\SI{17.5}{\degree} angle and the length of the horizontal tank dimension, that is excluding the sloping  beach, is~\SI{100}{\meter}. The structure we consider is an offshore gravity platform (Fig.~\ref{fig:domain}, right) and the dimensions of the model are based on prototypical values. In particular, the base width of the platform is~\SI{45}{\meter} with height~\SI{25}{\meter}, three columns with base diameter~\SI{10}{\meter} extend from the bottom platform and narrow to a~\SI{4}{\meter} width at height~\SI{42.5}{\meter}. To generate waves, we implemented a hinged-type wave maker utilizing the corresponding Bi\'esel transfer function~\cite{Frigaard1993} to relate the wave height to the stroke of the paddle. For the flap-type wave maker we use, the  Biésel   transfer function is given by
\begin{equation}
 \frac{H}{S_0} = \frac{2 \sinh(kh)(1 - \cosh(kh) + kh\sinh(kh)}{kh(\sinh(kh) \cosh(kh) + kh)}, 
\end{equation}
where $S_0$ is the stroke at the free surface, $H$ is the wave height in the far-field, $k$ is the wavenumber, and $h$ is the water depth.

Wind generated ocean waves are empirically described by their energy spectrum. Here, we consider irregular seas with JONSWAP spectral density:
\begin{equation}\label{eq:spectrum}
    S(f) = \frac{\alpha g^2}{(2\pi)^4 f^5} \exp\biggl[-\frac{5}{4} \biggl(\frac{f_p}{f}\biggr)^{4}\biggr] \cdot \gamma^{\exp\bigl[ \frac{-(f-f_p)^2}{2\delta^2 f_p^2}\bigr]},
\end{equation}
where $\delta = 0.07$ for $k \le k_0$ and $\delta = 0.09$ for $k > k_0$ and $f_p$ is the peak frequency. In the original formulation $\alpha$ is related to the fetch and the mean wind speed, however for offshore applications, especially in the North Sea,  the following modified version  is often adopted~\cite{Naess2013} $\alpha = 5.058 {H_s^2 f_p^4}(1 - 0.287 \log\gamma)$. 

\subsection{Additional cases}

Here we provide results for the momentum variable for two different spectra with parameters:
\begin{itemize}
    \item JONSWAP 1: $\alpha = 0.060, \gamma = 3.0, H_s = \SI{13.2}{\meter}$.
    \item JONSWAP 2: $\alpha = 0.100, \gamma = 2.0, H_s = \SI{15.7}{\meter}$.
\end{itemize}
These parameters are representative of North Sea conditions. The chosen peak wave period is $T_p =  \SI{10}{\second}$. For these spectra, in Fig.~\ref{fig:probthetajonswap} we plot the distribution of the wavefield parameterization $f_\theta$.

The quantities of interest in this problem are the forces and moments acting on the platform:
\begin{equation}
q_{f} =  \max_{t\in[0,T]} \lvert{ F_x(t) \rvert} \quad \text{and} \quad q_{m}  =  \max_{t\in[0,T]} \lvert{ M_y(t) \rvert}.
\end{equation}
The forces and moments are normalized by $k_0^3/\rho g$ and $k_0^2/\rho g$, respectively. 

\begin{figure}[htb!]
    \centering
    \includegraphics[width=0.65\linewidth]{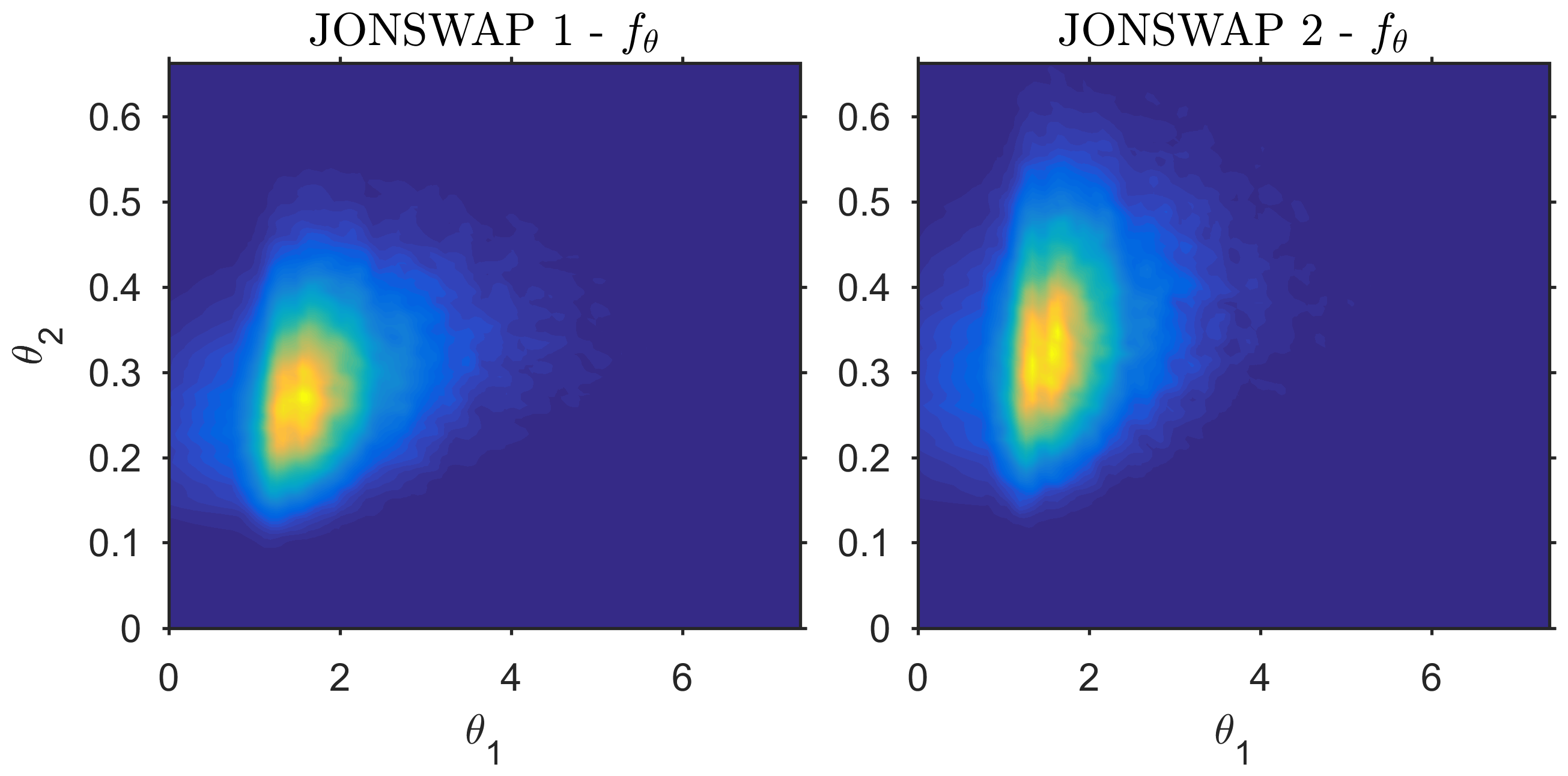}
    \caption{The parameter to observation map $T(\theta)$ for the force $F_x$ (left) and moment $M_y$ (right). Recall that $\theta_1$ is the lengthscale $L$ of a wave group and $\theta_2$ is the amplitude or height $A$ of the group.}
    \label{fig:probthetajonswap}
\end{figure}

In Fig.~\ref{fig:exact_sph_maps} we show the exact parameter to observation maps for the variables under consideration, which is computed by a cubic interpolation of 48 samples (i.e. numerical simulations of the SPH simulations for various $\theta$ values). It is important to note that the parameter to observation map \emph{remain fixed regardless of the spectra of the wave field.}  Fig~\ref{fig:exact_sph_pdfs1} display the target, i.e the true, pdf from the densely sampled map shown in Fig.~\ref{fig:exact_sph_maps} for both JONSWAP spectra.
\begin{figure}[htb!]
    \centering
    \includegraphics[width=0.625\linewidth]{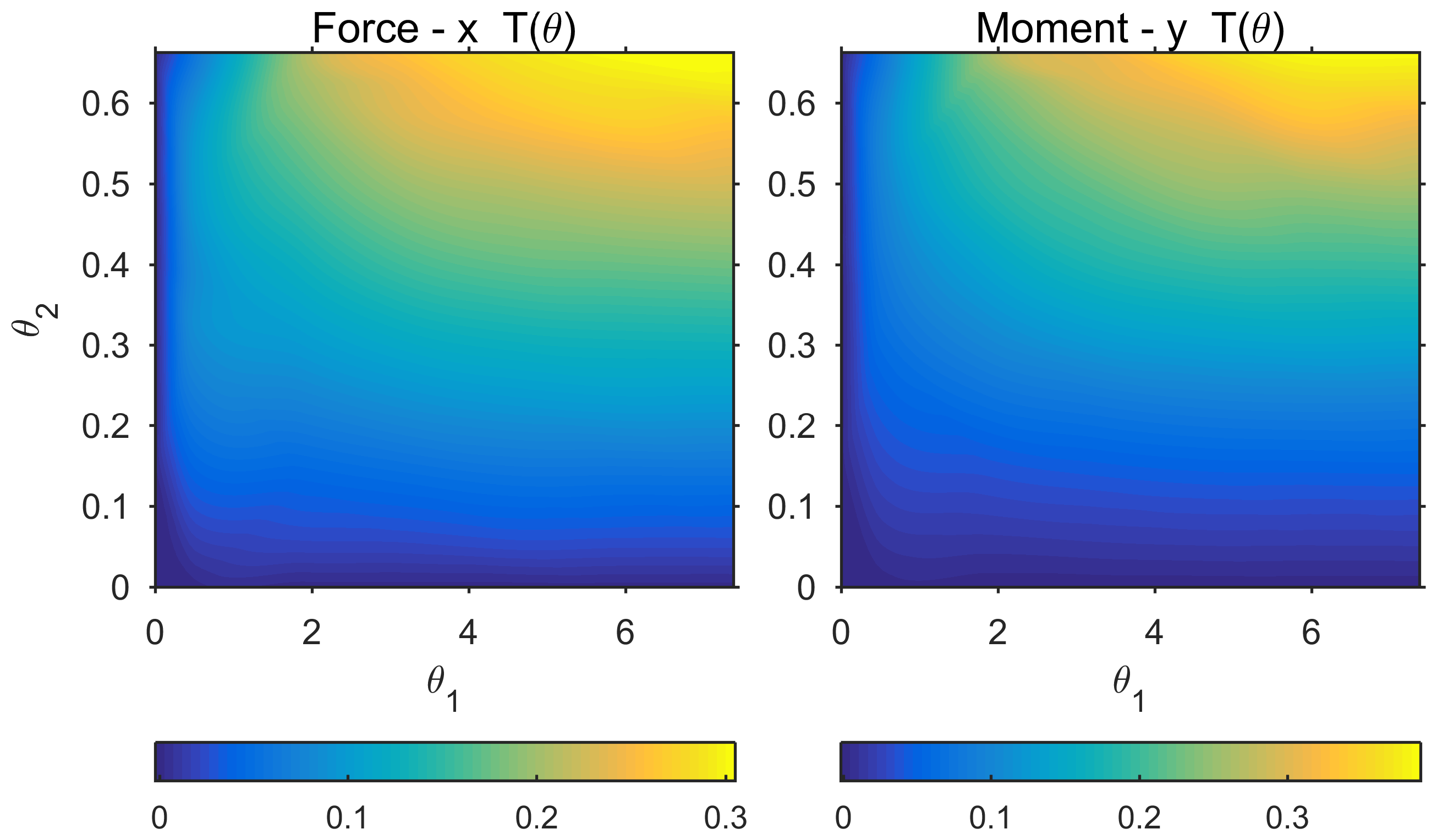}
    \caption{The parameter to observation map $T(\theta)$ for the force $F_x$ (left) and moment $M_y$ (right). Recall that $\theta_1$ is the lengthscale $L$ of a wave group and $\theta_2$ is the amplitude or height $A$ of the group.}
    \label{fig:exact_sph_maps}
\end{figure}
 \begin{figure*}[htb!]
    \centering

    \includegraphics{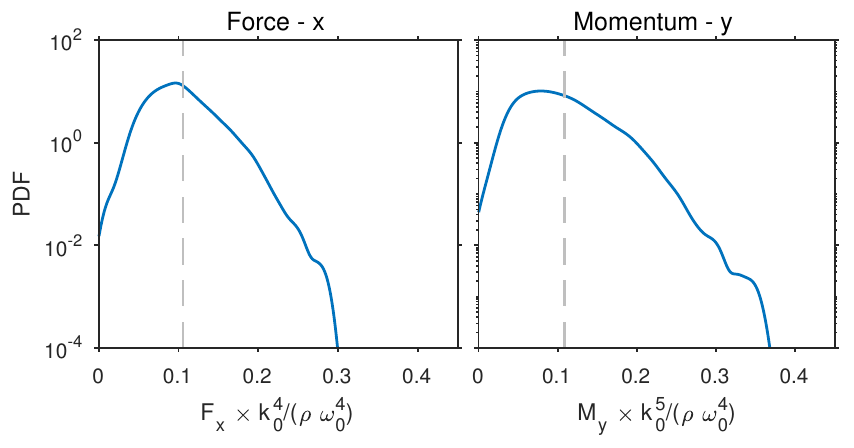}
    \includegraphics{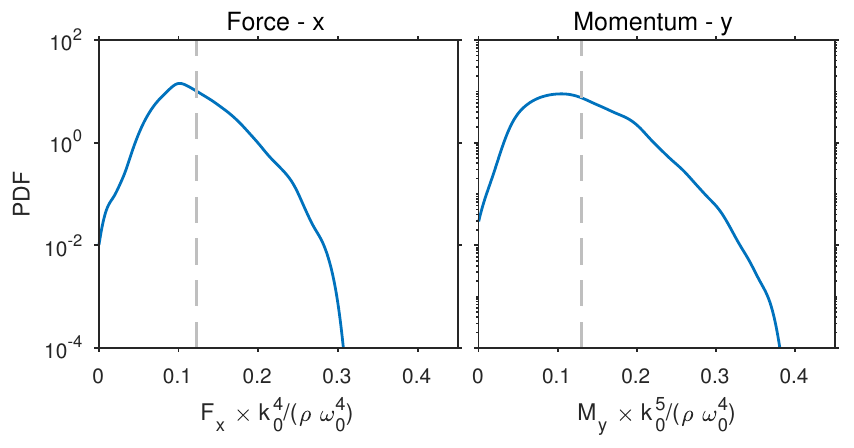}
    \caption{The exact pdf of the force and momentum variables (JONSWAP 1 left column and JONWAP 2 right column).}
    \label{fig:exact_sph_pdfs1}
\end{figure*}

In Fig.~\ref{fig:result_jonswap1_momentum} we show the result for the momentum variable for JONSWAP spectrum 1 (force variable results are provided in the main text). The results for JONSWAP spectrum 2 are summarized in Figs.~\ref{fig:result_jonswap2_force} and \ref{fig:result_jonswap2_momentum}.
\begin{figure*}[htb!]
    \centering
    \includegraphics{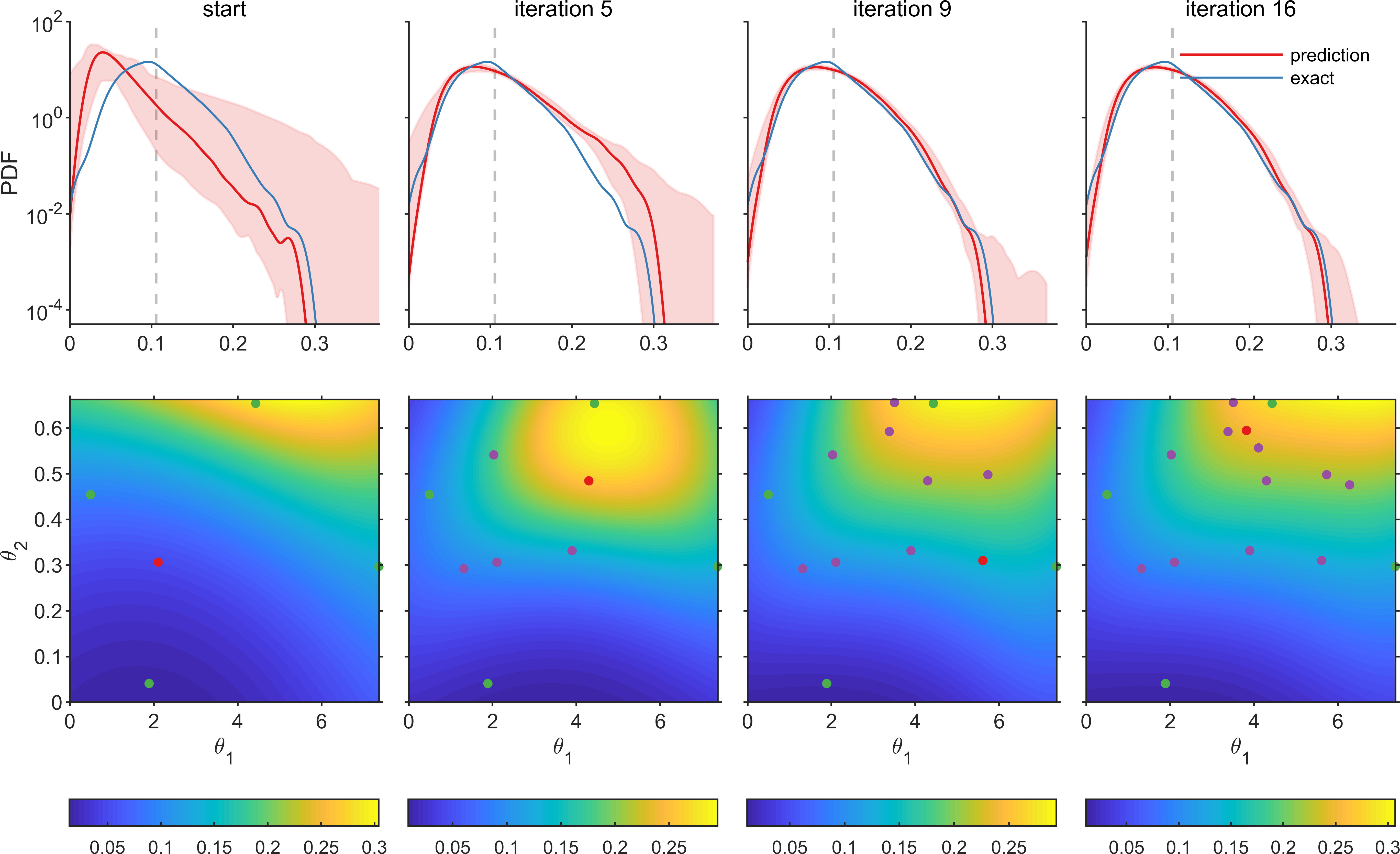}
    \caption{[JONSWAP 1]  Algorithm progression for the force variable.}
    \label{fig:result_jonswap1_force}
\end{figure*}
\begin{figure*}[htb!]
    \centering
    \includegraphics{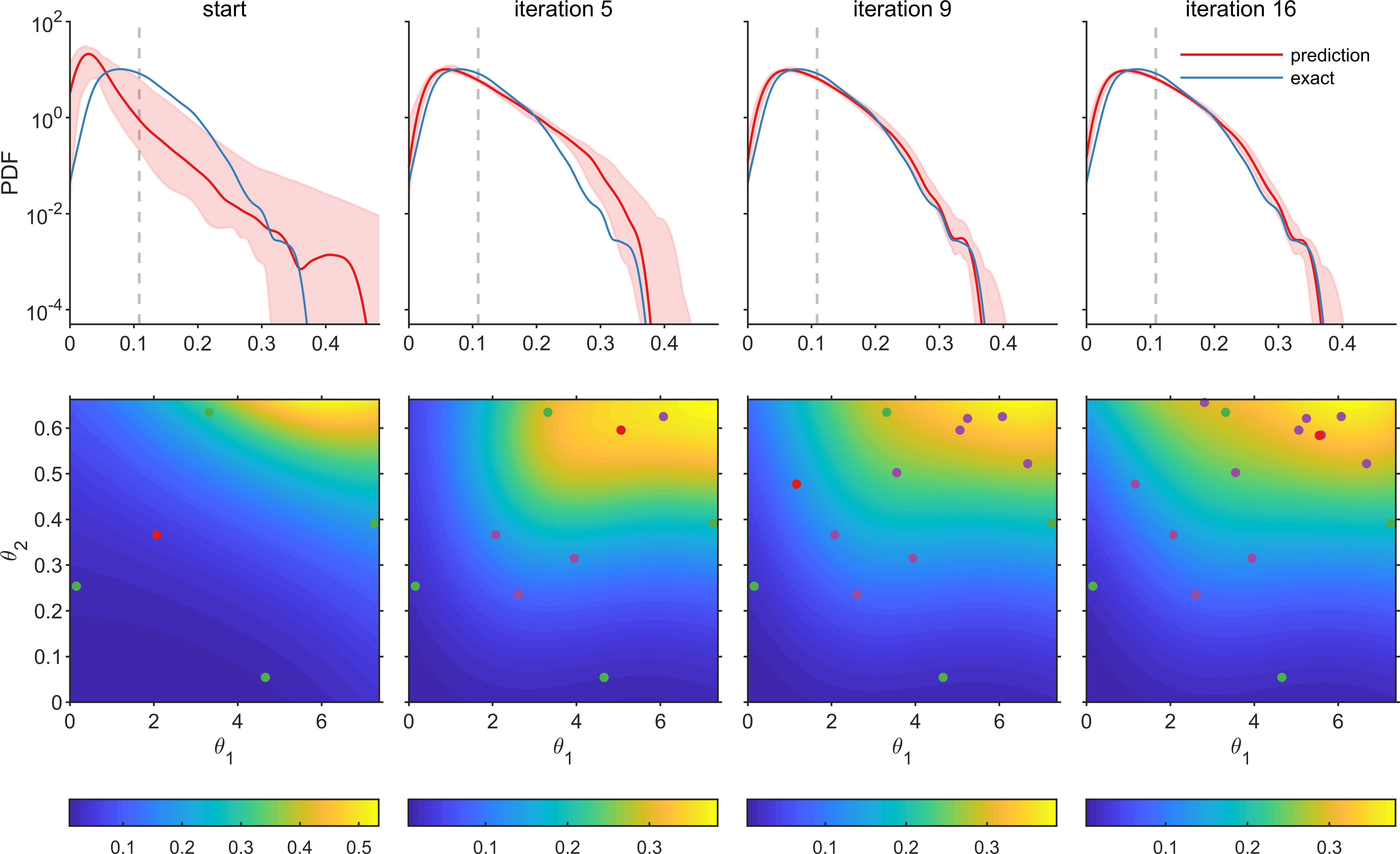}
    \caption{[JONSWAP 1]  Algorithm progression for the momentum variable.}
    \label{fig:result_jonswap1_momentum}
\end{figure*}
\begin{figure*}[htb!]
    \centering
    \includegraphics{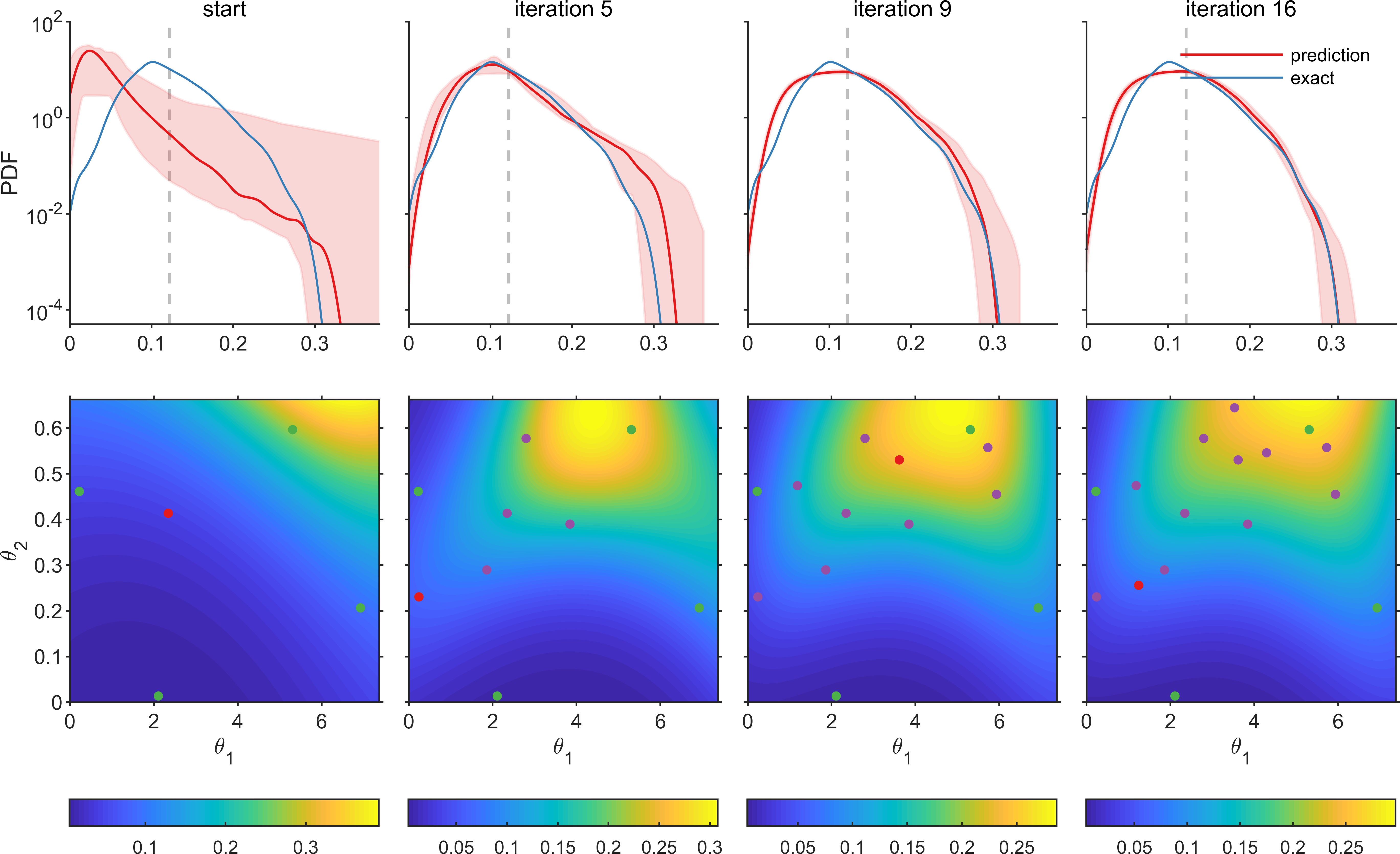}
    \caption{[JONSWAP 2] Algorithm progression for the force variable.}
    \label{fig:result_jonswap2_force}
\end{figure*}
\begin{figure*}[htb!]
    \centering
    \includegraphics{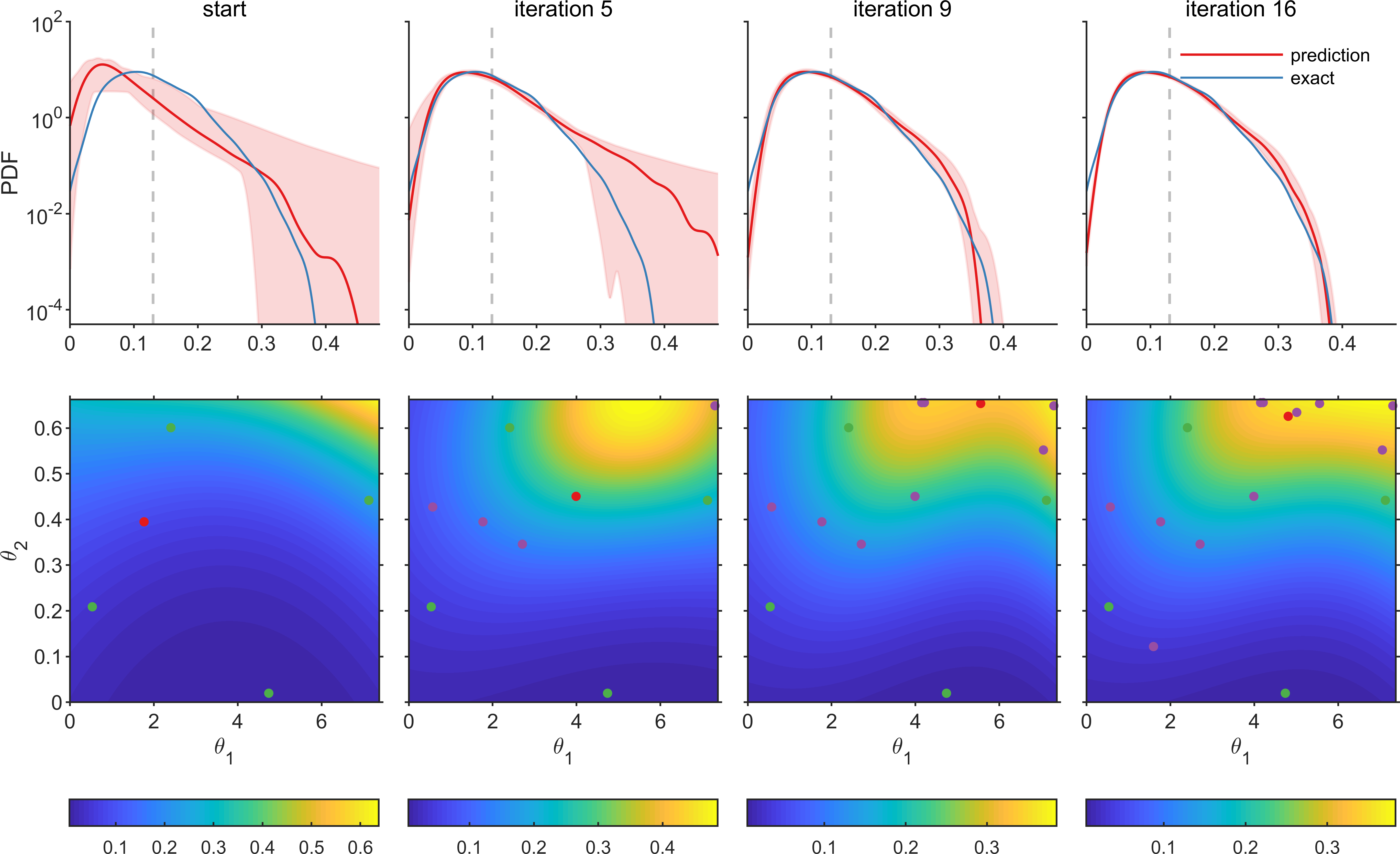}
    \caption{[JONSWAP 2] Algorithm progression for the momentum variable.}
    \label{fig:result_jonswap2_momentum}
\end{figure*}



\end{document}